\documentclass[12pt]{amsart}
\usepackage{latexsym, amsmath, amscd, amssymb, amsthm} 
\usepackage[T2A]{fontenc}
 \usepackage[cp1251]{inputenc}
\usepackage[czech, english]{babel}
\newtheorem{te}{Theorem}
 
 \newtheorem{pr}{Example}[section]

\def\mathbi#1{\textbf{\em #1}}

\begin{document}

\noindent

 \title[$3D$ geometric  moment invariants]{ 3D  geometric  moment invariants from the point of view of the classical  invariant theory }

\author{L. Bedratyuk}\address{Khmelnitskiy national university, Instytus'ka, 11,  Khmelnitskiy, 29016, Ukraine}

\begin{abstract}  
The aim of this paper is to clear up the problem of the connection between the 3D geometric moments invariants and the invariant theory, considering a problem of  describing  of the 3D geometric moments invariants as a problem of the classical invariant theory.
Using the remarkable fact that the groups  $SO(3)$ and $SL(2)$ are   locally isomorphic, we reduced the problem of  deriving 3D geometric moments invariants  to the well-known problem of the classical invariant theory.
 We give a precise statement of  the 3D geometric invariant moments computation, introducing the  notions of the algebras of simultaneous 3D geometric moment invariants, and   prove that   they are isomorphic  to the algebras of joint  $SL(2)$-invariants of several  binary forms. To simplify the calculating of the invariants we  proceed from an action of Lie group $SO(3)$ to  equivalent action of the  Lie algebra $\mathfrak{sl}_2$. The author hopes that the results will be useful to the researchers in the
fields of image analysis and pattern recognition.
\end{abstract}

\maketitle
%Keywords Pattern recognition · Feature descriptors · Geometric moment
%invariants · Rotation group · Classical invariant theory 

\section{Introduction}

	The issue of the 3D geometric moments is a 
generalization of  the 2D geometric moment invariants   which are widely  used as  global  feature descriptors in the different applications for pattern recognition and  image analysis. 
 Notice, that by  invariance we mean the  invariance with respect to translations, uniform scaling and rotations. In nowadays, the interest to the usage of the 3D moment invariants is stimulated by the rapid growth of the 3D technologies, \cite{PMCA}-\cite{KS}.

  For the first time, the 3D moment invariants of the second order were derived  in the paper  \cite{SH}.  In \cite{LD}, Lo and  Don  found  twelve invariants of the third order, but as it was shown in  \cite{S3d} there are several interdepended among them. In  the book \cite{FSB}, the author derived 13  invariants and stated that they generate all 3D  geometric moments of the third order. Finally, in   \cite {TST} a set of  one 1185 invariants up to  order 16 was presented, but  these invariants do not form a minimal generating system. However, finding a minimal generating system of the 3D geometric moment invariants  still remains an open problem. This kind of problems turn out to be a purely algebraic questions which were  studied widely in the 19th century.
	
Today, there exists a huge massive of the literature on the 3D geometric moments invariants, but a big amount of it is devoted to the application
of the invariants, along with the different ways of their constructions which sometimes are rather elegant and ingenious.
 For instance, the methods of the quantum mechanics used in   \cite{LD}, \cite{S3d} and \cite{BH} are very impressive.
	
	But, those methods  based on the rotation group $SO(3)$  are quite complicated and are not adapted well for the invariants calculations. 
In this paper, we propose to proceed from the usage of the $SO(3)$ group to the usage of its  locally isomorphic group $SL(2)$.
As far as the Lie algebras $\mathfrak{so}_3$  and $\mathfrak{sl}_ 2$ are isomorphic, the problem of finding of  $SO(3)$-invariants is equivalent to the problem of finding of $SL(2)$-invariants. The latter one is a well-known problem of the classical invariant theory issues, concequentely,  the standard classical invariant theory approaches can be applied.

 The aim of this paper is to consider the problem of describing 3D geometric moment invariants precisely as a problem of the classical invariant theory.   
We formulated the problem of the computation of the 3D geometric moments invariants based on the notion of the  algebras of the both rational and polynomial simultaneous  invariants of several binary forms. Our goal is not to find new invariants, we just  put together some facts about the  geometric 3D moments and presented it from a single point of view.  

In this article, we proved that the introduced algebras of the 3D geometric moment invariants are isomorphic to the well-known  objects of the classical invariant theory, namely, algebras of the joint   invariants of  the several binary forms.
In the rational case, we firstly applied the standard infinitesimal method to the studying of the geometric moments and reduced the problem of calculating the $SO(3)$-invariants  to the equivalent problem of calculating the invariants of its Lie algebra $\mathfrak{so}_3$.

The paper is arranged as follows. 

In  Sect.~2, we  review  basic concepts of the classical invariant theory and provide the necessary facts regarding the action of the Lie groups $SO(3)$ and $SL(2)$ and their Lie algebras $\mathfrak{so}_3, $  and $\mathfrak{sl}_2, $  respectively on the vector spaces of binary and ternary forms. We introduce the  notions of the algebras of simultaneous rational and polynomial   3D geometric moment invariants and  prove  that  they are isomorphic  to the algebras of joint  rational and polynomial $\mathfrak{sl}_2$-invariants of several  binary forms. Also, we presented a system of   partial differential equations concerning those   invariants.

In  Sect.~3, we recall the basic notions of the representation theory of  the Lie algebras and present a minimal generating system for the algebra of the 3D geometric polynomial moments invariants of orders two and three which is expressed in the terms of eigenvectors of the  Casimir operator.  Also   we derive the formula for the corresponding Poincar\'e series.

In  Sect.~4,   we count out the number of elements in a minimal generating set of the algebra rational rotation invariants and present such minimal generating set for the rational invariants of second and third orders. Also, we express the explicit form the invariants of the degrees one of arbitrary order.

The article is a continuation of the \cite {BB} article, which addresses the similar issues for the 2D geometric moment invariants.

%======================================================

\section{Preliminary Concepts}
In this section,
 we briefly review some basic concepts of the classical invariant theory, give the necessary facts about the Lie  groups  $SO(3), SL(2)$ and their Lie algebras   $\mathfrak{so}_3$ and $\mathfrak{sl}_2$.  
Also,  we give  the definition of the algebras of simultaneous  rational and polynomial  3D geometric moment invariants and then establish an isomorphism between these algebras and the algebras of  the  joint  invariants of several  binary forms.

%In this section, we briefly review some basic concepts of the classical invariant theory.
%We give the definition of the algebras of simultaneous  rational and polynomial rotational 2D geometric moment invariants and then establish an isomorphism between these algebras and the algebras of  the  join rotation invariants of  binary forms.
\subsection{Basic notions  of  the invariant theory }

%================================================
Let  $GL(V)$ be the group of all invertible linear transformations of a finite-dimensional  complex vector space $V.$ The natural action of  $GL(V)$ on $V$ produces an action on the algebras of polynomial and rational functions  $\mathbb{C}[V]$   and  $\mathbb{C}(V)$.    If  $g \in GL(V)$, $F \in \mathbb{C}[V]$ define a new polynomial function  $ g \cdot F \in \mathbb{C}[V]$ by 
$$
(g \cdot F)(v)=F\left(  g^{-1} v\right).
$$
If $G$ is subgroup of  $GL(V)$ we say that   $F$ is  \textit{$G$-invariant} if $ g \cdot F=F$ for all  $g \in G.$
The $G$-invariant polynomial functions forms a  
subalgebra $\mathbb{C}[V]^G$ of $\mathbb{C}[V]$. The algebra $\mathbb{C}[V]^G$  is called the \textit{algebra of the polynomial $G$-invariants}. In the similar way, we define the algebra of\textit{ rational invariants}  $\mathbb{C}(V)^G.$

Let us recall that \textit{a derivation} of an  algebra   $R$ is an additive  map  $L$ satisfying the Leibniz rule: 
$$
L(r_1 \, r_2)=L(r_1) r_2+r_1 L(r_2), \text{  for all }  r_1, r_2 \in R.
$$
The subalgebra
$$
\ker L:=\left \{ f \in R|  L(f)=0 \right \},
$$
is called \textit{the kernel} of the derivation $L.$

Let  now $G $ be  a simply connected Lie group acting on $V$    and  let $\mathfrak{g}$ be its  Lie algebra.  By  an action of $\mathfrak{g}$ we understand its  representation  by preserving  Lie products of linear   operators on   $V.$  We will extend  these operators   on   $\mathbb{C}[V]$ and  $\mathbb{C}(V)$ as derivatives. It is well known that the condition   $ I \in \mathbb{C}[V]^G$  is equivalent to  $L(I)=0$, $\forall L \in \mathfrak{g}.$ Thus, 
$$
\mathbb{C}[V]^G=\mathbb{C}[V]^\mathfrak{g}=\bigcap\limits_{L \in \mathfrak{g}} \ker L.
$$ 
As a linear object, a Lie algebra is often a  much easier to work with than working directly with the corresponding Lie group.
We will use this fact later to ease  the computation of invariants.

The classical invariant theory is focused on the action of the general linear group on  homogeneous
polynomials, with an emphasis on the forms, mainly  binary and ternary ones. 
Let us consider two important invariant constructions  which  illustrate a computational advantage of the Lie algebras techniques.

 \begin{pr}\label{binaryF}{\rm The space  $V_d$  of \textit{binary forms} of degree $d$ is the vector space:
$$
V_d=\left\{ \sum_{k=0}^d \binom{d}{k} a_{k} x^{d-k} y^k \mid  a_{k} \in \mathbb{C} \right\}.
$$
The goup   $SL(2)$ is   a group of   $2 \times 2$ complex matrices   with determinant  one. %  with the natural action on   
%$V_d$. 
The corresponding  Lie algebra  $\mathfrak{sl}_2 $  is generated by the
 matrices with zero trace  
\begin{gather*} h = \begin{pmatrix}
1 & 0\\
0 & -1
\end{pmatrix}, \quad 
e_+ = \begin{pmatrix}
0 & 1\\
0 & 0
\end{pmatrix}, \quad 
e_- = \begin{pmatrix}
0 & 0\\
1 & 0
\end{pmatrix} \end{gather*}
and the following commutation relations 
$$
 [h,e_+] = 2e_+, \quad [h,e_-] = -2e_-, \quad [e_+,e_-] = h.
$$

The elements $ e_-, e_+, h$   act on  $V_d$ by the derivations
$$
 -y \frac{\partial }{\partial x}, -x \frac{\partial }{\partial y}, -x \frac{\partial }{\partial x}-y \frac{\partial }{\partial y},
$$
and act on  $\mathbb{C}(V_d)$ by the derivations
$$
D_{+}=\sum_{k=0} (d-1) a_{k+1}\frac{\partial}{a_{k}},
D_{-}=\sum_{k=1}^d k a_{k-1}\frac{\partial}{a_{k}},
H=\sum_{k=0}^d (d-2k) a_{k}\frac{\partial}{a_{k}}.
$$
The polynomial solutions of the corresponding system of differential equations generate the algebra 
 $\mathbb{C}[V_d]^{\mathfrak{sl}_2}$ of invariants of binary form.  Since, 
$$[D_+, D_-]=D_+ D_- - D_- D_+=H$$
it follows that 
$$
\mathbb{C}[V_d]^{\mathfrak{sl}_2}=\ker D_+ \cap \ker D_-.
$$
The  \textit{minimal generating systems}  of $\mathbb{C}[V_d]^{\mathfrak{sl}_2}$ were  a major object of research in classical invariant theory of the 19th century. At present, 
such generators have been found only for $d \leq 10.$

 In the similar manner we define an action of  $SL(2)$ and  $\mathfrak{sl}_2$  on the direct sum 
$$
W= V_{k_1} \oplus V_{k_2} \oplus \cdots \oplus V_{k_n}.
$$
The corresponding algebras of polynomial and rational invariants are called the algebras of  \textit{joint invariants} (polynomial or rational)  of binary forms and   denoted by  $\mathbb{C}[W]^{\mathfrak{so}_2}$  and  $\mathbb{C}(W)^{\mathfrak{so}_2}$, respectively.
 At the present time, the algebras of the joint invariants are  only known for a few values  of $k_1, k_2, \ldots, k_n, $ see \cite{Olive}.

}
\end{pr}

\begin{pr}\label{so3}{\rm   
The 3D rotation group $SO(3)$ is the group of all rotations about the origin of three-dimensional Euclidean space. It is a three-parameters group with the following matrix realization  

$$
 \begin{pmatrix}  \cos\varphi & -\sin\varphi & 0\\ \sin\varphi &\cos\varphi & 0 \\ 0 & 0 & 1 \end{pmatrix}, \begin{pmatrix} \cos\theta & 0& -\sin\theta \\ 0 &1 & 0 \\ \sin\theta & 0 &\cos\theta  \end{pmatrix}, \begin{pmatrix}   \cos\psi & -\sin\psi & 0\\ \sin\psi &\cos\psi & 0 \\ 0 & 0 & 1  \end{pmatrix},  \psi, \theta, \varphi \in [0, 2 \pi].
$$ 
where the parameters   $\psi, \theta, \varphi$ are the \textit{Euler angles}.

 The associated tree-dimensional \textit{complex}  Lie algebra   $\mathfrak{so}_3$  is generated by the matrix
\begin{gather*}
e_1=\begin{pmatrix} 0 & 1 & 0 \cr -1 & 0 & 0 \cr 0 & 0 & 0 \end{pmatrix},
e_2=\begin{pmatrix} 0 & 0 & 1 \cr 0 & 0 & 0 \cr -1 & 0 & 0 \end{pmatrix},
e_3=\begin{pmatrix} 0 & 0 & 0 \cr 0 & 0 & 1 \cr 0 & -1 & 0 \end{pmatrix}, 
\end{gather*}
and the Lie brackets are given by commutator, i.e.,% (перевірено)
$$
[e_1,e_2]=-e_3,\;[e_1,e_3]=e_2,\;[e_2,e_3]=-e_1.
$$
Let us recall that the space of \textit{ternary forms} of degree $d$ is the vector space:
$$
T_d=\left\{ \sum_{j+k+l=d} \binom{d}{j,k,l} a_{j,k,l} x^j y^k z^l \mid  a_{j,k,l} \in \mathbb{C} \right\},
$$
 where $\displaystyle \binom{d}{j,k,l}=\frac{d!}{j! k! l!}$ denotes the multinomial coefficient. 
The linear functions  
$$
 \sum_{j+k+l=d} \binom{d}{j,k,l} a_{j,k,l} x^j y^k z^l \mapsto a_{j,k,l},
$$
form a basis of the dual vector space  $T_d^*$. For convenience, it is useful to  equal the functions and  the corresponding coefficients $a_{j,k,l}$.

It is a well-known, see, for example, \cite{Woit}, that  $\mathfrak{so}_3$  %eлементи $e_1,e_2,e_2$
 acts on  $T_d$  by derivations
$$
x \frac{\partial}{\partial y}-y \frac{\partial}{\partial x},x \frac{\partial}{\partial z}-z \frac{\partial}{\partial x}, y \frac{\partial}{\partial z}-z \frac{\partial}{\partial y}.
$$
Wherefrom, it follows that $\mathfrak{so}_3$ acts on the dual space  
 $T_d^*$ by the derivations:
\begin{te}\label{T1}
\begin{align*}
&E_1(a_{j,k,l})=k a_{j+1,k-1,l}-j a_{j-1,k+1,l},\\
&E_2(a_{j,k,l})=l a_{j+1,k,l-1} -j a_{j-1,k,l+1} ,\\
&E_3(a_{j,k,l})=l a_{j,k+1,l-1}-k a_{j,k-1,l+1}.
\end{align*}
\end{te}
\begin{proof}
Using a definition of the $\mathfrak{so}_3$-action on the dual space.
\end{proof}

In the similar manner, we define an action of $SO(3)$  and   $\mathfrak{so}_3$  on the direct sum 
$$
U=T_{k_1} \oplus T_{k_2}  \oplus \cdots \oplus T_{k_3}.
$$
The corresponding algebras of polynomial and rational invariants are called the algebras of  \textit{joint 3D rotation invariants}   and   denoted by  $\mathbb{C}[U]^{\mathfrak{so}_3}$  and  $\mathbb{C}(U)^{\mathfrak{so}_3}$, respectively. %The number $d$ is called the \textit{order} of an 3D invariant.
  More details about 3D rotations can be found , e.g.,  in \cite{Woit}, \cite{Hall}.

 }
\end{pr}

An important detail  that plays a crucial role in this article is a well-known fact that the complex Lie algebras $ \mathfrak {so}_3 $ and $\mathfrak {sl}_3 $ are isomorphic, although the corresponding Lie groups are not isomorphic. To establish the isomorphism, we introduce new matrices

 \begin{gather*}
\mathcal{D}_+=i e_1{+}e_2= \left( \begin {array}{ccc} 0&i&1\\ \noalign{\medskip}-i&0&0
\\ \noalign{\medskip}-1&0&0\end {array} \right),
\mathcal{D}_-=i e_1{-}e_2=\left( \begin {array}{ccc} 0&i&1\\ \noalign{\medskip}-i&0&0\\ \noalign{\medskip}-1&0&0\end {array} \right),
\mathcal{H}=2i e_3=2i \begin{pmatrix} 0 & 0 & 0 \cr 0 & 0 & 1 \cr 0 & -1 & 0 \end{pmatrix}.
\end{gather*}
By direct calculations of their commutators, we obtain 
$$
 [\mathcal{H},\mathcal{D}_+] = 2\mathcal{D}_+, \quad [\mathcal{H},\mathcal{D}_-] = -2\mathcal{D}_-, \quad [\mathcal{D}_+,\mathcal{D}_-] = \mathcal{H}.
$$
The commutators  coincide with the corresponding commutators of the basic elements  for the algebra
  $\mathfrak{sl}_2,$ which establishes the isomorphism. 

Note that the operators act on the basis elements of $T_d^*$ as follows

\begin{align*}
&\mathcal{D}_+(a_{j,k,l})=i \left(k a_{j+1,k-1,l}-j a_{j-1,k+1,l}\right)+l a_{j+1,k,l-1} -j a_{j-1,k,l+1},\\
&\mathcal{D}_-(a_{j,k,l})=i \left(k a_{j+1,k-1,l}-j a_{j-1,k+1,l}\right)-\left(l a_{j+1,k,l-1} -j a_{j-1,k,l+1} \right),\\
&\mathcal{H}(a_{j,k,l})=2i\left(l a_{j,k+1,l-1}-k a_{j,k-1,l+1} \right).
\end{align*}

	As we will  see later, this isomorphism allows  us  reduce the problem of finding the 3D rotation invariants to the problem of calculating the invariants of binary forms which is a  classical  invariant theory problem.
%=======================================================

\subsection{Algebras of 3D rotation invariants.}

%==========================================================

In the sequel, we will work with 
the  similarity transformation group  $G$ which  is widely used in    3D   image analysis and
pattern recognition. The group     is  the semi-direct product of the space translation group  $TR(3)$, the   direct product of the  space rotation group $SO(3)$  and the  uniform scaling  group $\mathbb{R}^*$:   
$$
G=(\mathbb{R}^* \times SO(3)) \rtimes TR(3). 
$$
%Also,    $G$ is a four-parametric group with the following action on   $(x,y) \in \mathbb{R}^2:$ 
%$$
%(x,y)^{\top} \mapsto  s \begin{pmatrix}  \cos\theta & -\sin\theta\\ \sin\theta &\cos\theta \end{pmatrix} \begin{pmatrix}  x\\ y \end{pmatrix}+\begin{pmatrix}  a\\ b \end{pmatrix}.
%$$

The introduction  of the notion of  2D image moment invariants by Hu in the  significant paper \cite{Hu}  is a   
vivid  example of the application of the classical invariant theory to the pattern recognition. A way of the generalization of this approach for 3D images was suggested in \cite{SH}, \cite{LD}.   Let  $\mathbi{F}$ be a set of   real finite piece-wise continuous functions  that  can have nonzero values only in 
 a compact subset of  $\mathbb{R}^3.$ 
  
		Let us consider   \textit{ the geometric moments } of  $f \in \mathbi{F}$
		
		$$
m_{pqr}(f(x,y,z))=m_{pqr}=\iiint\limits_{\Omega} x^p y^q z^r f(x,y,z) dx dy dz, \Omega \subset \mathbb{R}^3,
$$
and the \textit{central geometric moment}
$$
\mu_{pqr}(f(x,y,z))=\mu_{pqr}=\iiint\limits_{\Omega} (x-\bar{x})^p (y-\bar{y})^q (z-\bar{z})^r f(x,y,z) dx dy dz,
$$
where
$$
\bar{x}=\frac{m_{100}}{m_{0,0,0}},\bar{y}=\frac{m_{010}}{m_{0,0,0}}, \bar{y}=\frac{m_{001}}{m_{0,0,0}}
$$
The central geometric moments are already invariants under the translation group.  After the  normalization 
$$
\eta_{p,q,r}=\frac{\mu_{p,q,r}}{\mu_{0,0}^{1+\frac{p+q+r}{3}}}, p+q+r \geq 2,
$$
they become invariants  of the scaling group. Therefore, the problem of determining of the 3D geometric image moment invariants   can be  reduced to the  problem of finding $SO(3)$-invariants as functions of the normalized central geometric moments.   Therefore, in this paper we will deal   only  with the normalized $SO(3)$-invariant functions.

We will  consider two types of such functions, specifically,
 polynomials and rational ones. 
Let  $\mathbb{C}[\eta]$ and $\mathbb{C}(\eta)$ be the  polynomial and rational  algebras   in  countably many  variables $\{\eta_{p,q,r}\}_{p+q+r=2}^{\infty}$ considered with the natural action of the group  $SO(3).$
Denote  by   $\mathbb{C}[\eta]^{SO(3)}$ and  $\mathbb{C}(\eta)^{SO(3)}$ the corresponding algebras of \textit{polynomial}  and  \textit{rational  moment } invariants, respectively.  Since these algebras are not finitely generated, then  a complete set of invariants consists of infinitely many invariants. However, these algebras can be approximated by the finitely generated algebras $\mathbb{C}[\eta]_d^{SO(3)}$ and  $\mathbb{C}(\eta)_d^{SO(3)}$where $[\eta]_d=\{ \eta_{p,q,r},2 \leq p+q+r \leq d\} $. The elements of these algebras are called the \textit{simultaneous}  3D geometric moment (polynomial or rational) invariants    of \textit{order} up to $d.$ 
 For instance, the invariant
$$ \eta_{2,0,0}+\eta_{0,2,0}+\eta_{0,0,2},$$
belong to $\mathbb{C}[\eta]_2^{SO(3)}$ and  $\mathbb{C}(\eta)_2^{SO(3)}.$

		Remarkably, in general case, the problem of describing the algebras of the simultaneous 3D geometric moment invariants 
can be reduced to the well-known problems of the classical
invariant theory.  It turns out that 
the algebras  $\mathbb{C}(\eta)_d^{SO(3)}$ and  $\mathbb{C}(\eta)_d^{SO(3)}$ are isomorphic   to the algebras  of \textit{joint  polynomial  and rational $SL(2)$-invariants} of  some  system of binary forms. 

The \textit{locally isomorphism} of $SO(3)$ and  $SL(2)$ implies the following theorem.

\begin{te}
The algebras of polynomial and rational simultaneous  3D geometric moment  invariants $\mathbb{C}[\eta]_d^{SO(3)}$   and  $\mathbb{C}(\eta)_d^{SO(3)}$ are isomorphic to the algebras of  invariants $\mathbb{C}[U_d]^{\mathfrak{sl}_2}$  and   $\mathbb{C}(U_d)^{\mathfrak{sl}_2}$, respectively.   Here 
$$
U_d=T_2 \oplus T_3 \oplus \cdots \oplus T_d, 
$$
and   $T_k$  is the vector space  of  ternary forms of order  $k$.
\end{te}
\begin{proof} It is sufficient to check that the algebras  
 $\mathfrak{so}_3$ and  $\mathfrak{sl}_2$  act by identical derivatives.
Let us consider the action of the element $$\begin{pmatrix}  \cos\theta & -\sin\theta & 0\\ \sin\theta &\cos\theta & 0 \\ 0 & 0 & 1 \end{pmatrix} \in SO(3)$$
 on the  normalized moment $\eta_{j,k,l}$. By the definition, we have   

\begin{gather*}
\begin{pmatrix}  \cos\theta & -\sin\theta & 0\\ \sin\theta &\cos\theta & 0 \\ 0 & 0 & 1 \end{pmatrix}^{-1}\eta_{p,q,r}=\iiint\limits_{\Omega} (x\cos\theta-y\sin\theta)^p (x\sin\theta +y\cos\theta)^q z^r f(x,y,z) dx dy dz=\\=
\iiint\limits_{\Omega} \sum_{k=0}^p \sum_{j=0}^q (-1)^{p-k}\binom{p}{k}\binom{q}{j} (\cos\theta)^{p-k+j} (\sin\theta)^{q+k-j} x^{p-k+q-j} y^{k+j} z^r f(x,y,z) dx dy dz=\\
=\sum_{k=0}^p \sum_{j=0}^q (-1)^{p-k}\binom{p}{k}\binom{q}{j} (\cos\theta)^{p-k+j} (\sin\theta)^{q+k-j}\eta_{p-k+q-j, k+j,r}.
\end{gather*}
To get the action of the Lie algebra  $\mathfrak{so}_3$ we differentiate it by $\theta$ and, after simplification, we obtain:
$$
\frac{d}{d \theta} \begin{pmatrix}  \cos\theta & -\sin\theta & 0\\ \sin\theta &\cos\theta & 0 \\ 0 & 0 & 1 \end{pmatrix}^{-1} \eta_{p,q,r} \Big|_{\theta=0}=q \eta_{p+1,q-1,r}-p \eta_{p-1,q+1,r}.
$$

It is easy to see that this action is identical to the derivation  $E_1$, as it described in Example~\ref{so3}.

In the same manner,  we can show that the following elements  of $SO(3):$

$$
  \begin{pmatrix}  1 & 0 & 0 \\ 0 &\cos\theta & -\sin\theta \\ 0 & \sin\theta &\cos\theta  \end{pmatrix}, \begin{pmatrix}   \cos\psi & -\sin\psi & 0\\ \sin\psi &\cos\psi & 0 \\ 0 & 0 & 1  \end{pmatrix}, 
$$ 
 act like  the derivations $E_2$ and $E_3.$
Thus, the normalized 3D geometric moment invariants and the joint  $\mathfrak{so}_3$-invariants of the  binary forms are defined by the same  system of  the  partial differential equation. It implies that  $\mathbb{C}(\eta)_d^{SO(3)} \cong \mathbb{C}(U_d)^{\mathfrak{so}_3}$  and $\mathbb{C}[\eta]_d^{SO(3)} \cong \mathbb{C}[U_d]^{\mathfrak{so}_3}.$
Since,  $\mathfrak{so}_3 \cong \mathfrak{sl}_2$  we get that $\mathbb{C}(\eta)_d^{SO(3)} \cong \mathbb{C}(U_d)^{\mathfrak{sl}_2}$  and $\mathbb{C}[\eta]_d^{SO(3)} \cong \mathbb{C}[U_d]^{\mathfrak{sl}_2}$  as required.

The isomorphism has the simple form: $a_{j,k,r} \mapsto \eta_{j,k,r}.$
\end{proof}

Thus, from the point of view of  the classical invariant theory, the problem of  the description  of the algebras 3D geometric image  moment invariants $\mathbb{C}[\eta]_d^{SO(3)}$,  $\mathbb{C}(\eta)_d^{SO(3)}$ can be reduced to the following two problems.

\begin{itemize}
	\item 
\textbf{Problem~1.}  What is  a minimal generating set of  the algebra polynomial joint invariants  $\mathbb{C}[U_d]^{\mathfrak{sl}_2}?$

\item 
\textbf{Problem~2.} What is  a minimal generating set of  the algebra rational  joint invariants  $\mathbb{C}(U_d)^{\mathfrak{sl}_2}$?
\end{itemize}

%There is a more general result, namely  the \textit{fundamental  theorem  of  moment  invariants}, compare схожі результат для 2D моментних інварінтів \cite{Hu} and \cite{RR}, \cite{MM}, which  can be now reformulated  as follows:

%\begin{te} For any linear algebraic group $G \subseteq GL(2)$ considered with the standard action on $\mathbb{C}(\eta)$  the simultaneous 3D geometric moment invariants  algebra   $\mathbb{C}(\eta)_d^{G}$ is isomorphic to the  algebra  $\mathbb{C}(U_d)^{G}$ of the \textit{joint invariants} several binary forms. 
%\end{te}

Besides, the problem of deriving of 3D geometric moment  invariants can be  reduced to a system of differential equations. The last result of Subsest.~2.1 implies the  theorem:
\begin{te} The algebra  $\mathbb{C}(U_d)^{\mathfrak{sl}_2}$ coincides with the algebra of rational solutions of the first order simultaneous  partial differential equations:
$$
\begin{cases}
\displaystyle \sum_{2 \leq j+k+l\leq d}(k \eta_{j+1,k-1,l}-j \eta_{j-1,k+1,l})\frac{\partial U }{\partial \eta_{j,k,l}}=0,\\
\displaystyle \sum_{2 \leq j+k+l\leq d} (l \eta_{j+1,k,l-1} -j \eta_{j-1,k,l+1})\frac{\partial U }{\partial \eta_{j,k,l}}=0.
\end{cases}
$$
\end{te}

In the next section we will deal with the algebras   $\mathbb{C}[U_d]^{\mathfrak{sl}_2}$ and $\mathbb{C}(U_d)^{\mathfrak{sl}_2}$.

%\subsection{Зв'язок з ортогональними моментами} 

%==========================================================================================

    \section{The algebra of    polynomial  invariants $\mathbb{C}[U_d]^{\mathfrak{sl}_2}$. }

		%===========================================================================================
		
		Let us recall some facts about representations of the Lie algebra 
 $\mathfrak{sl}_2.$
\subsection{ Representations  of $\mathfrak{sl}_2$}

Let  $V$ be a finite-dimensional complex vector  space equipped with  non-trivial linear operators   $D_+,D_-,H : V \to V,$ which satisfy the following commutation relations 
$$
[H,D_+] = H D_+-D_+ H= 2D_+, \quad [H,D_-] = -2D_-, \quad [D_+,D_-] = H
$$
Then $V$  is called the  \textit{linear representation} of the Lie algera $\mathfrak{sl}_2$ or $\mathfrak{sl}_2$-\textit{module}. The vector spaces $T_k^*$, $U_d$ defined above are the samples of    $\mathfrak{sl}_2$-{modules.} The modules $0$ and $V$ are called \textit{trivial} modules.
A $\mathfrak{sl}_2$-module $V$  is called  \textit{irreducible} if $V$ has no non-trivial  $\mathfrak{sl}_2$-\textit{submodule}. All  irreducible  $\mathfrak{sl}_2$-modules, up to isomorphism,  can be described with  the following construction.

Let  $\mathcal{V}_n=\langle a_0,a_1,\ldots a_n \rangle$  be a  $n{+}1$-dimension complex vector space and let the linear operators $D_-, D_+, H: V_n \to V_n$   act on elements of the  basis as follows : 
\begin{align*}
&D_-(a_k)=k a_{k-1}, D_+ (a_k)=(n-k) a_{k+1}, H(a_k)=(n-2k) a_k.
\end{align*} 
 Let us check that the commutation relation for $\mathfrak{sl}_2$ are fulfilled. In fact, we have
\begin{align*}
&[D_-, D_+](a_k)=D_-(D_+ (a_k))-D_+( D_- (a_k))=D_-((d-k) a_{k+1})-D_+( k a_{k-1})=\\&=(d-k)(k+1)a_k-k(d-(k-1))a_k=(d-2k)a_k=H(a_k),\\
 &[H,  D_-](a_k)=H(D_-(a_k))-D(H(a_k))=H(k a_{k-1})-D((d-2k)a_k)=\\&=k(d-2(k-1))a_{k-1}-(d-2k)k a_{k-1}=2 k a_{k-1}=2 D(a_k),\\
&[H,  D_+](a_k)=H(D_+(a_k))-D_+(H(a_k))=H((d-k) a_{k+1})-D_+((d-2k)a_k)=\\&=
(d-k)(d-2(k+1)) a_{k+1}-(d-2k)(d-k)a_{k+1}=-2(d-k)a_{k+1}=-2 D_+(a_k).
\end{align*}
Therefore, $\mathcal{V}_n$ is an representation of  $\mathfrak{sl}_2.$ The vector space $\mathcal{V}_n$ considered together with the indicated action of the operators $D_-, D_+, H$ is called \textit{the standard irreducible }$sl_2$-module.
It is well known, see \cite{FH}, the an arbitrary  $\mathfrak{sl}_2$-module can be  decomposed into  an direct sum of the irreducible standard   $\mathfrak{sl}_2$-modules.    
Next,  we present an algorithm of decomposing an arbitrary $\mathfrak{sl}_2 $-module into the irreducible submodules.  We  use the algorithm later to construct invariants.

 Let $W$ be an arbitrary $\mathfrak{sl}_2$-module.
%Оператори $D_+$ та $Н$ визначають дві важливі числові функції на $sl_2$-модулі $W$  -- \textit{порядок} та \textit{вагу. }
For any element $w \in W$ the smallest natural number, denoted ${\rm ord}(w)$, such that 
$$
D_+^{{\rm ord}(z) }(w) \ne 0, \mbox{  but  } D_+^{{\rm ord}(z)+1}(w) = 0. 
$$
is called \textit{the order} of $w.$ Since 
 $D_+$ is a \textit{nilpotent} operator,  the order ${\rm ord}(w)$ is defined correctly.

A vector $z \in W$ is called  \textit{the lowest weight vector}  if the following conditions holds: $D_-(z)=0$  and  $H(z)=\mbox{ord}(z)z.$  Any  lowest weight vector defines an irreducible  $sl_2$-module  which is isomorphic to the standard  $sl_2$-module.  The following  theorem holds.

\begin{te}\label{sb}
Suppose   $z \in W$ is a  lowest weight vector. Then the vector space 
$$  
{\mathcal{V}}_s(z){:=}\langle v_0(z),v_1(z),\ldots v_{s}(z) \rangle,  s={\rm ord}(z),
$$ 
where  $$v_k(z)=\frac{(s-k)!}{ s!}\, D_+^{\,{k}}(z), v_0(z):=z,$$ 
is $\mathfrak{sl}_2$-module   isomorphic to the standard  $sl_2$-module $\mathcal{V}_s$.
\end{te}
\begin{proof}

It is easy to verify by direct calculations that the  relations  
\begin{align*}
 & H(D_+^{\,k}(z))=(s-2k)\,D_+^{\,k}(z),\\
 & D_-(D_+^{\,k}(z))=k (s-k+1)\,D_+^{\,k-1}(z),  
\end{align*}
hold for all 
 $k \leq s.$
 Let us construct  the  standard $\mathfrak{sl}_2$-module  $\mathcal{V}_s$ with the basis vectors of the form 
$$
v_k=\alpha_{k} D_+^k(z), k=0,\ldots,s,
$$
for some unknown constants $\alpha_{k} \in \mathbb{C}.$

In order the vectors  form  a   basis of $\mathcal{V}_s$,  the following two conditions must be satisfied:
$$
D_-(v_k)=k v_{k-1}, D_+ (v_k)=(s-k) v_{k+1},
$$
for all $k=0,\ldots,s.$
Since 
$$
D_-(v_k)=D_-(\alpha_k D_+^k(z))=\alpha_k D_-(D_+^k(z))=\alpha_k k (s-k+1)\,D_+^{\,k-1}(z),
$$
and 
$$
D_-(v_k)=k v_{k-1}=k \alpha_{k-1}D_+^{\,k-1}(z), 
$$
 we obtain the  recurrence equation for  $\alpha_k:$
$$
\alpha_k (s-k+1)=\alpha_{k-1}, \alpha_0=1.
$$
It follows immediately that 
$$
\alpha_k=\frac{1}{s (s-1) \ldots (s-k+1)} \, \alpha_0 =\frac{(s-k)!}{s!}.
$$
Let us make sure that the second relation 
$$
D_+ (v_k)=(s-k)v_{k+1},
$$
also holds.
We have 
\begin{gather*}
D_+ (v_k)=D_+(\alpha_k  D_+^k(v_0))=\frac{(s-k)!}{s!} \frac{s!}{(s-(k+1))!}\frac{(s-(k+1))!}{s!}  D_+^{k+1}(v_0)=(s-k)v_{k+1},
\end{gather*}
as required which ends the proof.
\end{proof}

The theorem  below  determines   a structure of the  $sl_2$-modules $T^*_d$ and $U_d$  up to isomorphism:
\begin{te}\label{main}  The following decompositions hold:
\begin{align*}
&T^*_d \cong \mathcal{V}_{2d} \oplus \mathcal{V}_{2d-4} \oplus \mathcal{V}_{2d-8} \oplus \cdots \oplus \mathcal{V}_{2d-4\left[\frac{d}{2} \right]},\\
&U^*_d \cong l^{(d)}_0 \mathcal{V}_0 \oplus l^{(d)}_1\mathcal{V}_{2} \oplus l^{(d)}_2\mathcal{V}_{4} \oplus \cdots \oplus l^{(d)}_d \mathcal{V}_{2d},
\end{align*} where 
$$
l^{(d)}_k=\begin{cases}  0, \text{if } k>d,\\
\left[\dfrac{d{-}k}{2} \right], \text{ if $k=0,1$,}\\
\left[\dfrac{d-k}{2} \right]+1, \text{ if $k \geq 2$}.\\
  \end{cases}
$$
\end{te}
Since the proof requires some advanced results of the Lie algebras representation theory,  we omit the proof.

\begin{pr}  For small $d$  we have
\begin{align*}
&U_2^*=T^*_2 \cong \mathcal{V}_0 \oplus \mathcal{V}_4,\\
&U_3^*=T^*_2+T^*_3 \cong \mathcal{V}_0 \oplus \mathcal{V}_2 \oplus \mathcal{V}_4 \oplus \mathcal{V}_6,\\
&U_4^*=T^*_2+T^*_3+T^*_4 \cong 2\mathcal{V}_0 \oplus \mathcal{V}_2 \oplus 2\mathcal{V}_4 \oplus \mathcal{V}_6 \oplus \mathcal{V}_8.
\end{align*}
\end{pr}
%ОТже, суми тернарних форм, можна заінити сумами бінарних орм
\begin{pr}{\rm 
Theorem~\ref{main} implies that the invariant of degree one exist only in the case of even  $d$. We   can  write an explicit form for these invariants.
For  any $d=2m$, we consider the element
$$
I_{d}=\sum_{j+k+l=m} \binom{m}{j,k,l} a_{2j,2k,2l}.
$$
It is an invariant if 
 the following conditions hold
$$
E_1(I_d)=E_2(I_d)=E_1(I_d)=0.
$$
Let us prove that 
 $E_1(I_d)=0.$  We have 
\begin{gather*}
E_1(I_d)=\sum_{j+k+l=d} \binom{d}{j,k,l} E_1(a_{2j,2k,2l})=\sum_{j+k+l=d} \binom{d}{j,k,l}(2k \, a_{2j+1,2k-1,2l}-2j\,a_{2j-1,2k+1,2l}).
\end{gather*}
Then
\begin{gather*}
\sum_{j+k+l = n} k\binom{n}{j,k,l}a_{2j+1,2k-1,2l}=
\sum_{\substack{j+k+l = n \\ k > 0}} k\binom{n}{j,k,l}a_{2j+1,2k-1,2l} =\\
 \sum_{\substack{j+k+l = n \\ k > 0}} n\binom{n-1}{j,k-1,l}a_{2j+1,2k-1,2l} \stackrel{s = k-1}{=}
 \sum_{j + s + l = n-1} n \binom{n-1}{j,s,l} a_{2j+1,2s+1,2l}=\\ %\tag{$s = k-1$}
 \stackrel{m = j+1}{=}\sum_{\substack{m + s + l = n \\ m > 0}} n\binom{n-1}{m-1,s,l} a_{2m-1,2s+1,2l} =%\\%\tag{$m = j+1$}
 \sum_{m + s + l = n} m\binom{n}{m,s,l} a_{2m-1,2s+1,2l} =\\ 
\stackrel{{\small \begin{array}{c}j = m, \\k = s \end{array}}}{=} \sum_{j + k + l = n} j\binom{n}{j,k,l} a_{2j-1,2k+1,2l}\,.% \tag{$j = m, k = \ell$}
\end{gather*}

Thus
$$
\sum_{j+k+l = n} k\binom{n}{j,k,l}a_{2j+1,2k-1,2l}=\sum_{j + k + l = n} j\binom{n}{j,k,l} a_{2j-1,2k+1,2l},
$$
and $E_1(I_d)=0.$  %Here we shifted  the indexes of summation 
In the same way,  we can show  that $
E_2(I_d)=0
$  and $E_3(I_d)=0.$

For small $d$ we have 
\begin{gather*}
{I}_2=a_{{0,0,2}}+a_{{0,2,0}}+a_{{2,0,0}},
\\
{I}_4=a_{{0,0,4}}+2\,a_{{0,2,2}}+a_{{0,4,0}}+2\,a_{{2,0,2}}+2\,a_{{2,2,0}}+a_{{4,0,0}},\\
{I}_6=3\,a_{{4,0,2}}+3\,a_{{4,2,0}}+a_{{6,0,0}}+3\,a_{{0,4,2}}+a_{{0,6,0}}+3\,a_{{2,0,4}}+6\,a_{{2,2,2}}+3\,a_{{2,4,0}}+a_{{0,0,6}}+3\,a_{{0,2,4}},\\
{I}_8=6\,a_{{4,4,0}}+4\,a_{{6,0,2}}+4\,a_{{6,2,0}}+a_{{8,0,0}}+12\,a_{{2,4,2}}+4\,a_{{2,6,0}}+6\,a_{{4,0,4}}+12\,a_{{4,2,
2}}+a_{{0,8,0}}+4\,a_{{2,0,6}}+\\+12\,a_{{2,2,4}}+a_{{0,0,8}}+4\,a_{{0,2,
6}}+6\,a_{{0,4,4}}+4\,a_{{0,6,2}}
\end{gather*}

}
\end{pr}
Theorem~\ref{main} implies
\begin{te} The following decompositions hold:
$$
\begin{array}{lll}
&(i) & \mathbb{C}[\eta]_d^{SO(3)} \cong \mathbb{C}[l^{(d)}_0 V_0 \oplus l^{(d)}_1V_{2} \oplus l^{(d)}_2V_{4} \oplus\cdots \oplus l^{(d)}_d V_{2d}]^{\mathfrak{sl}_2},\\
&(ii) & \mathbb{C}(\eta)_d^{SO(3)} \cong \mathbb{C}(l^{(d)}_0 V_0 \oplus l^{(d)}_1V_{2} \oplus l^{(d)}_2V_{4} \oplus \cdots \oplus l^{(d)}_d V_{2d})^{\mathfrak{sl}_2}
\end{array}
$$
\end{te}

Therefore, it implies that the problem of determining   of the algebra 3D geometric polynomial and rational  moment invariants is  equivalent to the problem of 
of determining  of the algebras joint $\mathfrak{sl}_2$-invariants.
It appears to be a very difficult problem in terms of performing calculations and it is quite a challenge to find a minimal  generating set for  $d> 5$.

\subsection{ The algebra of 3D polynomial moment  invariants  $\mathbb{C}[\eta]_2^{SO(3)}$. }\label{ss2}

Let us illustrate the above with the  references to the algebra of  3D polynomial moment  invariants  of order two.
  Since Theorem~\ref{main} implies that   $T_2^* \cong \mathcal{V}_0^* \oplus \mathcal{V}_4^*$, the algebra of 3D polynomial moment  invariants  $\mathbb{C}[\eta]_2^{SO(3)}$
$\mathbb{C}[\eta]_2^{{SO(3)}}$ is equal to the algebra of  $\mathfrak{sl}_2$-invariants  $\mathbb{C}[\mathcal{V}_0(u_0) \oplus \mathcal{V}_4(v_0)]^{\mathfrak{sl}_2},$ where  $u_0, v_0$ are the  lowest weight vectors in the   $T^*_2$-\textit{realizations} of the standard  $\mathfrak{sl}_2$-modules $\mathcal{V}_0$ and $\mathcal{V}_4$.

To find such a realization, firstly we need  to   find the realization  of the standard basis of 
$\mathcal{V}_0$ and  $\mathcal{V}_4$ on  $T^*_2$  and, then, substitute it into the expressions for the generating invariants of the algebra  $\mathbb{C}[\mathcal{V}_0(u_0)\oplus\mathcal{V}_4(v_0)]^{\mathfrak{sl}_2}$. 
Since  $\mathcal{V}_0$ is a trivial ${\mathfrak{sl}_2}$-module, it is enough to find the generating elements of the alebra  $\mathbb{C}[V_4(v_0)]^{\mathfrak{sl}_2}$. But $\mathbb{C}[V_4(v_0)]^{\mathfrak{sl}_2}$ which is isomorphic to the  classical algebra of  invariants  of binary form of degree four
$$
a_{{0}}{x}^{4}+4\,a_{{1}}{x}^{3}y+6\,a_{{2}}{x}^{2}{y}^{2}+4\,a_{{3}}x
{y}^{3}+a_{{4}}{y}^{4}.
$$
It is well-known, that the latter is generated by the following two invariants of degree two and three:
\begin{align*}
&S_1=a_0a_4+3 a_2^2 - 4 a_1a_3,\\
&S_2= a_0a_2a_4 + 2 a_1a_2a_3- a_2^3-  a_0a_3^2 -  a_1^2a_4=\begin{vmatrix} a_0 & a_1 & a_2 \\
a_1 & a_2 & a_3 \\ a_2 & a_3 & a_4 \end{vmatrix}.
\end{align*}
In terms  of the classical invariant theory,  the invariant $S_1$ is called the \textit{apolar invariant} and the invariant $S_2$   is known as the the \textit{catalecticant}  or the \textit{Hankel determinant}.

The six-dimensional $sl_2$-module  $T_2^*$  is generated by the following elements: 
$$
T_2^*=\langle  a_{{0,0,2}},a_{{0,1,1}},a_{{0,2,0}},a_{{1,0,1}},a_{{1,1,0}},a_{{2,0,0}}  \rangle.
$$
The operators $\mathcal{D}_+,\mathcal{D}_-,\mathcal{H}$ act on the basis as follows (see Theorem~\ref{T1}):
\begin{align*}
&\mathcal{D}_+(a_{j,k,l})=i \left(k a_{j+1,k-1,l}-j a_{j-1,k+1,l}\right)+l a_{j+1,k,l-1} -j a_{j-1,k,l+1},\\
&\mathcal{D}_-(a_{j,k,l})=i \left(k a_{j+1,k-1,l}-j a_{j-1,k+1,l}\right)-\left(l a_{j+1,k,l-1} -j a_{j-1,k,l+1} \right),\\
&\mathcal{H}(a_{j,k,l})=2i\left(l a_{j,k+1,l-1}-k a_{j,k-1,l+1} \right).
\end{align*}

The  lowest weight vectors  $u_0, v_0$ of the $\mathfrak{sl}_2$-modules $\mathcal{V}_0(u_0)$ and  $\mathcal{V}_4(v_0)$ are the  solutions of the following two  systems  of linear equations: $\begin{cases}\mathcal{E}_-(z)=0,\\ \mathcal{H}(z)=0 \end{cases}$ and  $ \begin{cases} \mathcal{E}_-(z)=0, \\ \mathcal{H}(z)=4z  \end{cases}$, respectively. Thus, we  obtain
\begin{align*}
&u_0=I_1=a_{{0,0,2}}+a_{{0,2,0}}+a_{{2,0,0}},\\
&v_0=2\,a_{{0,1,1}}+i ( a_{0,0,2}-a_{{0,2,0}} ).
\end{align*}
The element  $u_0$ is already an invariant.

Using Theorem~\ref{sb}, we get the standard basis $\mathcal{V}_4(v_0)$:
\begin{align*}
&v_0=v_0=2\,a_{{0,1,1}}+i ( a_{0,0,2}-a_{{0,2,0}} ),\\
&v_1=\frac{1}{4}D_+(x_0)=ia_{{1,0,1}}+a_{{1,1,0}},\\
&v_2=\frac{1}{12}D_+^2(x_0)=-\frac{i}{3} \left( a_{{0,0,2}}+a_{{0,2,0}}-2\,a_{{2,0,0}} \right) ,\\
&v_3=\frac{1}{24}D_+^3(x_0)=a_{{1,1,0}}-ia_{{1,0,1}},\\
&v_4=\frac{1}{24}D_+^4(x_0)=-2\,a_{{0,1,1}}+i \left( -a_{{0,2,0}}+a_{{0,0,2}} \right) 
\end{align*}

Substituting  $v_i$ for $a_i$   in  $S_1, S_2$ we find  invariants  $I_2$ and $I_3:$

\begin{align*}
I_2=&{a_{{0,0,2}}}^{2}{-}a_{{0,0,2}}a_{{0,2,0}}{-}a_{{0,0,2}}a_{{2,0,0}}+3\,{a_
{{0,1,1}}}^{2}+{a_{{0,2,0}}}^{2}{-}a_{{0,2,0}}a_{{2,0,0}}+3\,{a_{{1,0,1}
}}^{2}+3\,{a_{{1,1,0}}}^{2}+{a_{{2,0,0}}}^{2},\\
I_3=&2\,{a_{{0,0,2}}}^{3}-3\,{a_{{0,0,2}}}^
{2}a_{{2,0,0}}+9\,a_{{0,0,2}}{a_{{0,1,1}}}^{2}-3\,{a_{{0,2,0}}}^{2}a_{
{0,0,2}}+12\,a_{{0,0,2}}a_{{2,0,0}}a_{{0,2,0}}+\\
&+9\,a_{{0,0,2}}{a_{{1,0,
1}}}^{2}-18\,a_{{0,0,2}}{a_{{1,1,0}}}^{2}-3\,a_{{0,0,2}}{a_{{2,0,0}}}^
{2}+9\,{a_{{0,1,1}}}^{2}a_{{0,2,0}}-18\,{a_{{0,1,1}}}^{2}a_{{2,0,0}}+
\\ &+
54\,a_{{0,1,1}}a_{{1,1,0}}a_{{1,0,1}}-3\,{a_{{0,2
,0}}}^{2}a_{{2,0,0}}-18\,a_{{0,2,0}}{a_{{1,0,1}}}^{2}+9\,a_{{0,2,0}}{a
_{{1,1,0}}}^{2}-3\,a_{{0,2,0}}{a_{{2,0,0}}}^{2}+\\
&+9\,{a_{{1,0,1}}}^{2}a_
{{2,0,0}}+9\,{a_{{1,1,0}}}^{2}a_{{2,0,0}}+2\,{a_{{2,0,0}}}^{3}-3\,{a_{{0,0,2}}}^{2}a_{{0,2,0}}+2\,{a_{{0,2,0}}}^{3}.
\end{align*}

Thus, we have proved the following theorem.
 
\begin{te} The algebras of polynomial and rational invariants  $\mathbb{C}[T_2]^{\mathfrak{sl}_2}$, $\mathbb{C}(T_2)^{\mathfrak{sl}_2}$ are generated by the invariants  $I_1, I_2$ and $I_3:$ 
\begin{align*}
&\mathbb{C}[T_2]^{sl_2}=\mathbb{C}[I_1, I_2,I_3], \\
&\mathbb{C}(T_2)^{sl_2}=\mathbb{C}(I_1, I_2,I_3).
\end{align*}
\end{te}
In  order to obtain  the 3D moment invariant it  is sufficient to replace  $a_{j,k,l}$ by the normalized moments  $\eta_{j,k,l}$ in $I_1, I_2$ and $I_3$.

We should admit that the obtained result is confirmed by the result of \cite{SH}, \cite{LD}, \cite{FSB}
obtained by a different method.

As far as the obtained expressions for the invariants are quite cumbersome, we are interested in finding a simpler representation for them. Let us consider   \textit{the Laplace operator:}
$$
\mathcal{L}=\mathcal{D}_+ \mathcal{D}_-+\mathcal{D}_-\mathcal{D}_++\frac{1}{2} \mathcal{H}^2=E_1^2+E_2^2+E_3^2,
$$
which belongs to the \textit{enveloping algebra} of  $\mathfrak{sl}_2$. It can be proved that  $\mathcal{L}$ commutes  with  the operators $\mathcal{D}_+, \mathcal{D}_-, \mathcal{H}$. Therefore, $\mathcal{L}$ is  \textit{diagonalizable}  on every standard  $\mathfrak{sl}_2$-module. 

Let us express the invariants in terms of  the eigenvectors of the  Laplace operator
 $\mathcal{L}$. The operator $\mathcal{L}$ acts on the  basis of  $T^*_2$  as follows:
\begin{align*}
&\mathcal{L}(a_{{0,0,2}})=-4\,a_{{2,0,0}}+8\,a_{{0,0,2}}-4\,a_{{0,2,0}},\mathcal{L}(a_{{0,1,1}})=12\,a_{{0,1,1}},\\ 
&\mathcal{L}(a_{{0,2,0}})=8\,a_{{0,2,0}}-4\,a_{{2,0,0}}-4\,a_{{0,0,
2}},\mathcal{L}(a_{{1,0,1}})=12\,a_{{1,0,1}},\\
&\mathcal{L}(a_{{1,1,0}})=12\,a_{{1,1,0}},\mathcal{L}(a_{{2,0,0})
}=-4\,a_{{0,2,0}}+8\,a_{{2,0,0}}-4\,a_{{0,0,2}}.
\end{align*}
Since $T^*_2=\mathcal{V}_0(u_0) \oplus \mathcal{V}_4(v_0)$, then there exists one   eigenvector, let us denote it by  $e_0$ , associated  with the zero eigevalue and five  eigenvectors $e_1, e_2, e_3, e_4,e_5 $ associated  with the eigevalue twelve.The eigenvectors could be found by the standard linear algebra algorithm:
\begin{align*}
&e_0=a_{{0,0,2}}+a_{{0,2,0}}+a_{{2,0,0}},\\
&e_1=a_{{0,1,1}},e_2=a_{{0,2,0}}-a_{{0,0,2}},
e_3=a_{{1,0,1}},e_4=a_{{1,1,0}},e_5=a_{{2,0,0}}-a_{{0,0,2}}.
\end{align*}

Then the  invariants  $I_1,I_2$ and $I_3$ are expressed  in a much more compact form:

\begin{align*}
&I_1=e_0,\\
&I_2=3\,e_1^{2}+e_2^{2}-e_5\,e_2+3\,e_3^{2}+
3\,e_4^{2}+e_5^{2},\\
&I_3=9e_1^{2}e_2{-}18e_1^{2}e_5{+}54e_1e_4e_3{+}2e_2^{3}{-}3e_5e_2^{2}{-}18e_3^{2}e_2{+}9e_4^{2}e_2{-}3e_5^{2}e_2{+}9e_3^{2}e_5{+}9e_4^{2}e_5{+}2e_5^{3}.
\end{align*}

\subsection{The algebra of polynomial invariants  $\mathbb{C}[U_3]^{\mathfrak{sl}_2}$ }

We note, that the  case $d=3$ is much more complicated than the  case $d=2$. For $d=3$ we have the following decomposition of the ${\mathfrak{sl}_2}$-module $U_3:$ 
$$
U_3=T_2^* \oplus T_3^*\cong\mathcal{V}_0(v_0)  \oplus \mathcal{V}_2(x_0) \oplus \mathcal{V}_4(y_0) \oplus \mathcal{V}_6(u_0).
$$
Suppose the  $\mathfrak{sl}_2$-modules $\mathcal{V}_0, \mathcal{V}_2, \mathcal{V}_4, \mathcal{V}_6$ are given by their standard bases:
\begin{align*}
&\mathcal{V}_0(v_0)=\langle v_0\rangle,\\
&\mathcal{V}_2(x_0)=\langle x_0,x_1,x_2\rangle,\\
&\mathcal{V}_4(y_0)=\langle y_0,y_1,y_2,y_3,y_4\rangle,\\
&\mathcal{V}_6(u_0)=\langle u_0,u_1,u_2,u_3,u_4, u_5, u_6\rangle.\\
\end{align*}
Proceeding as above, we again find the lowest weight vectors $v_0, x_0, y_0, u_0$ by solving  systems of linear equations. Further, by using Theorem~\ref{sb} we obtain the following realization of  all these ${\mathfrak{sl}_2}$-modules in $U_3:$
\begin{align*}
 &v_{{0}}=a_{{0,0,2}}+a_{{0,2,0}}+a_{{2,0,0}},\\
 &x_{{0}}=a_{{0,0,3}}+a_{{0,2,1}}+a_{{2,0,1}}{-}i \left(a_{{0,1,2}}+a_{{0,3,0}}+a_{{2,1,0}} \right) ,x_{{1}}=a_{{1,0,2}}+a_{{1,2,0}}+
a_{{3,0,0}}, x_{{2}}=-\overline{x}_0,
\\
&y_{{0}}\!=\!2a_{{0,1,1}}{+}i \left(a_{{0,0,2}}{-}a_{{0,2,0}}\right) ,y_{{1}}\!=\!a_{{1,1,0}}{+}ia_{{1,0,1}},y_{{2}}\!=\!\frac{i}{3}\! \left( 2\,a_{2,0,0}{-}\!a_{{0,0,2}}{-}\!a_{{0,2,0}} \right) , y_{{3}}\!=\!\overline{y}_1,  y_{{4}}\!=\!-\overline{y}_0,\\
&u_{{0}}=a_{{0,0,3}}{-}3\,a_{{0,2,1}}+i \left(a_{{0,3,0}}{ -}3\,a_{{0,1,2}}\right) ,u_{{1}}=a_{{1,0,2}}-a_{{1,2,0}}-2\,ia_{{1,1,1}}, \\&u
_{{2}}=\frac{1}{5}(4\,a_{{2,0,1}}-a_{{0,0,3}}-\,a_{{0,2,1}}+i \left( a_{{0,1,2}}+a_{{0,3,0}}-4\,a_{{2,1,0}} \right)) , \\ &u_{{3}}=\frac{1}{5}(2\,a_{{3,0,0}}-3
\,a_{{1,0,2}}-3\,a_{{1,2,0}}), u_{{4}}=-\overline{u}_2 ,  u_{{5}}=\overline{u}_1,u_{{6}}=-\overline{u}_0.
\end{align*}
Here  $\overline{\phantom{a^b}}$ indicates  the complex conjugate.

Recently,
the  minimal generating set of polynomial invariants for the algebra  $\mathbb{C}[\mathcal{V}_2 \oplus \mathcal{V}_4 \oplus \mathcal{V}_6]^{\mathfrak{sl}_2}$ was calculated, see \cite{Olive},   in the symbolic form. The minimal generating set consists of 195 invariants and their degrees grow up to fifteen.  Therefore, a minimal  generating set of polynomial invariants  of the algebra   $\mathbb{C}[T_3]^{\mathfrak{sl}_2}$ consists of  196 invariants. 
These invariants can be calculated explicitly using author's Maple package \cite{Map} or by expanding the transvectants listed in the paper \cite{Olive}.  Below we present only the first thirteen invariants:
\begin{center}
\begin{tabular}{|l|l|l|}
\hline
{\rm deg} & Invariants &\#\\
\hline
1      &  $ v_0$ &1       \\
\hline
2      &   $x_{{0}}x_{{2}}-{x_{{1}}}^{2},
y_{{0}}y_{{4}}-4\,y_{{1}}y_{{3}}+3\,{y_{{2}}}^{2},
u_{{0}}u_{{6}}-6\,u_{{1}}u_{{5}}+15\,u_{{2}}u_{{4}}-10\,{u_{{3}}}^{2}$ &3       \\
\hline
3      &    $y_{{0}}y_{{2}}y_{{4}}-y_{{0}}{y_{{3}}}^{2}-{y_{{1}}}^{2}y_{{4}}+2\,y_{
{1}}y_{{2}}y_{{3}}-{y_{{2}}}^{3},$  &  4   \\

 & ${x_{{0}}}^{2}y_{{4}}-4\,x_{{0}}x_{{1}}y_{{3}}+2\,x_{{0}}x_{{2}}y_{{2}}
+4\,{x_{{1}}}^{2}y_{{2}}-4\,x_{{1}}x_{{2}}y_{{1}}+{x_{{2}}}^{2}y_{{0}},$ & \\

 & $u_{{0}}u_{{4}}y_{{4}}-2\,u_{{0}}u_{{5}}y_{{3}}-4
\,u_{{1}}u_{{3}}y_{{4}}+6\,u_{{1}}u_{{4}}y_{{3}}-2\,u_{{1}}u_{{6}}y_{{
1}}+3\,{u_{{2}}}^{2}y_{{4}}-4\,u_{{2}}u_{{3}}y_{{3}}-$ & \\
&$-9\,u_{{2}}u_{{4}}
y_{{2}}+6\,u_{{2}}u_{{5}}y_{{1}}+u_{{2}}u_{{6}}y_{{0}}+8\,{u_{{3}}}^{2
}y_{{2}}-4\,u_{{3}}u_{{4}}y_{{1}}-4\,u_{{3}}u_{{5}}y_{{0}}+3\,{u_{{4}}
}^{2}y_{{0}}+u_{{0}}u_{{6}}y_{{2}},$ & \\
&$u_{{0}}x_{{2}}y_{{4}}-2\,u_{{1}}x_{{1}}y_{{4}}-4\,u_{{1}}x_{{2}}y_{{3}
}+8\,u_{{2}}x_{{1}}y_{{3}}+6\,u_{{2}}x_{{2}}y_{{
2}}-4\,u_{{3}}x_{{0}}y_{{3}}-12\,u_{{3}}x_{{1}}y_{{2}}-$ & \\&$-4\,u_{{3}}x_{{2
}}y_{{1}}+6\,u_{{4}}x_{{0}}y_{{2}}+8\,u_{{4}}x_{{1}}y_{{1}}+u_{{4}}x_{
{2}}y_{{0}}-4\,u_{{5}}x_{{0}}y_{{1}}-2\,u_{{5}}x_{{1}}y_{{0}}+u_{{6}}x
_{{0}}y_{{0}}+u_{{2}}x_{{0}}y_{{4}}$ & \\
\hline
4& $u_{{0}}{x_{{2}}}^{3}-6\,u_{{1}}x_{{1}}{x_{{2}}}^{2}+3\,u_{{2}}x_{{0}}{
x_{{2}}}^{2}+12\,u_{{2}}{x_{{1}}}^{2}x_{{2}}-12\,u_{{3}}x_{{0}}x_{{1}}
x_{{2}}-8\,u_{{3}}{x_{{1}}}^{3}+$&5 \\
& $+3\,u_{{4}}{x_{{0}}}^{2}x_{{2}}+12\,u_{
{4}}x_{{0}}{x_{{1}}}^{2}-6\,u_{{5}}{x_{{0}}}^{2}x_{{1}}+u_{{6}}{x_{{0}
}}^{3},$& \\
& $u_{{0}}u_{{4}}{x_{{2}}}^{2}-2\,u_{{0}}u_{{5}}x_{{1}}x_{{2}}+u_{{0
}}u_{{6}}{x_{{1}}}^{2}-4\,u_{{1}}u_{{3}}{x_{{2}}}^{2}+6\,u_{{1}}u_{{4
}}x_{{1}}x_{{2}}+2\,u_{{1}}u_{{5}}x_{{0}}x_{{2}}-
$& \\
& $-2\,u_{{1}}u_{{5}}{x_{
{1}}}^{2}-2\,u_{{1}}u_{{6}}x_{{0}}x_{{1}}+3\,{u_{{2}}}^{2}{x_{{2}}}^{2
}-4\,u_{{2}}u_{{3}}x_{{1}}x_{{2}}-8\,u_{{2}}u_{{4}}x_{{0}}x_{{2}}-
u_{{2}}u_{{4}}{x_{{1}}}^{2}+$& \\
& $+6\,u_{{2}}u_{{5}}x_{{0}}x_{{1}}{+}u_{{2}
}u_{{6}}{x_{{0}}}^{2}+6\,{u_{{3}}}^{2}x_{{0}}x_{{2}}{+}2\,{u_{{3}}}^{2}
{x_{{1}}}^{2}{-}4\,u_{{3}}u_{{4}}x_{{0}}x_{{1}}{-}4\,u_{{3}}u_{{5}}{x_{{0}
}}^{2}{+}3\,{u_{{4}}}^{2}{x_{{0}}}^{2},$  & \\
&$u_{{0}}y_{{1}}{y_{{4}}}^{2}-3\,u_{{0}}y_{{2}}y_{{3}}y_{{4}}+2\,u_{{0}}
{y_{{3}}}^{3}-u_{{1}}y_{{0}}{y_{{4}}}^{2}-2\,u_{{1}}y_{{1}}y_{{3}}y_{{
4}}+9\,u_{{1}}{y_{{2}}}^{2}y_{{4}}-6\,u_{{1}}y_{{2}}{y_{{3}}}^{2}+
$& \\
&$+5\,u
_{{2}}y_{{0}}y_{{3}}y_{{4}}-15\,u_{{2}}y_{{1}}y_{{2}}y_{{4}}+10\,u_{{2
}}y_{{1}}{y_{{3}}}^{2}-10\,u_{{3}}y_{{0}}{y_{{3}}}^{2}+10\,u_{{3}}{y_{
{1}}}^{2}y_{{4}}-5\,u_{{4}}y_{{0}}y_{{1}}y_{{4}}+$ & \\
& $+15\,u_{{4}}y_{{0}}y_{
{2}}y_{{3}}-10\,u_{{4}}{y_{{1}}}^{2}y_{{3}}+u_{{5}}{y_{{0}}}^{2}y_{{4}
}+2\,u_{{5}}y_{{0}}y_{{1}}y_{{3}}-9\,u_{{5}}y_{{0}}{y_{{2}}}^{2}+6\,u_
{{5}}{y_{{1}}}^{2}y_{{2}}-u_{{6}}{y_{{0}}}^{2}y_{{3}}+$ & \\
& $+3\,u_{{6}}y_{{0}
}y_{{1}}y_{{2}}-2\,u_{{6}}{y_{{1}}}^{3},$ & \\
& ${x_{{0}}}^{2}y_{{2}}y_{{4}}-{x_{{0}}}^{2}{y_{{3}}}^{2}-2\,x_{{0}}x_{{1
}}y_{{1}}y_{{4}}+2\,x_{{0}}x_{{1}}y_{{2}}y_{{3}}+2\,x_{{0}}x_{{2}}y_{{
1}}y_{{3}}-2\,x_{{0}}x_{{2}}{y_{{2}}}^{2}+{x_{{1}}}^{2}y_{{0}}y_{{4}}-
$ & \\
& $-{x_{{1}}}^{2}{y_{{2}}}^{2}-2\,x_{{1}}x_{{2}}y_{{0}}y_{{3}}+2\,x_{{1}}x
_{{2}}y_{{1}}y_{{2}}+{x_{{2}}}^{2}y_{{0}}y_{{2}}-{x_{{2}}}^{2}{y_{{1}}
}^{2},$ & \\
& $u_{{0}}u_{{2}}u_{{4}}u_{{6}}-u_{{0}}u_{{2}}{u_{{5}}}^{2}-u_{{0}}{u_{{3
}}}^{2}u_{{6}}+2\,u_{{0}}u_{{3}}u_{{4}}u_{{5}}-u_{{0}}{u_{{4}}}^{3}-{u
_{{1}}}^{2}u_{{4}}u_{{6}}+{u_{{1}}}^{2}{u_{{5}}}^{2}+$
 & \\
& $+2\,u_{{1}}u_{{2}}
u_{{3}}u_{{6}}-2\,u_{{1}}u_{{2}}u_{{4}}u_{{5}}-2\,u_{{1}}{u_{{3}}}^{2}
u_{{5}}+2\,u_{{1}}u_{{3}}{u_{{4}}}^{2}-{u_{{2}}}^{3}u_{{6}}+2\,{u_{{2}
}}^{2}u_{{3}}u_{{5}}+$ & \\
& $+{u_{{2}}}^{2}{u_{{4}}}^{2}-3\,u_{{2}}{u_{{3}}}^{2
}u_{{4}}+{u_{{3}}}^{4}.$ &\\
\hline
      \end{tabular}
\end{center}

Substituting the realizations of the standard  $\mathfrak{sl}_2$-modules in the invariants expressions,  we get the explicit expressions for the invariants of the algebra   $\mathbb{C}[T_3]^{\mathfrak{sl}_2}$.
In  order to obtain  the 3D moment invariant it  is sufficient to replace  $a_{j,k,l}$ by the normalized moments  $\eta_{j,k,l}$.
 For example, the 3D geometric moment invariants of low degrees have the form 

\begin{align*}
B_0=&\eta_{{0,0,2}}+\eta_{{0,2,0}}+\eta_{{2,0,0}},\\
B_1=&{\eta_{{0,0,2}}}^{2}{-}\eta_{{0,0,2}}\eta_{{0,2,0}}{-}\eta_{{0,0,2}}\eta_{{2,0,0}}{+}3\,{\eta_
{{0,1,1}}}^{2}{+}{\eta_{{0,2,0}}}^{2}{-}\eta_{{0,2,0}}\eta_{{2,0,0}}+3\,{\eta_{{1,0,1}
}}^{2}+3\,{\eta_{{1,1,0}}}^{2}{+}{\eta_{{2,0,0}}}^{2},\\
B_2=&{\eta_{{0,0,3}}}^{2}+2\,\eta_{{0,0,3}}\eta_{{0,2,1}}+2\,\eta_{{0,0,3}}\eta_{{2,0,1
}}+{\eta_{{0,1,2}}}^{2}+2\,\eta_{{0,1,2}}\eta_{{0,3,0}}+2\,\eta_{{0,1,2}}\eta_{{2,
1,0}}+{\eta_{{0,2,1}}}^{2}+\\
&+2\,\eta_{{0,2,1}}\eta_{{2,0,1}}+{\eta_{{0,3,0}}}^
{2}+2\,\eta_{{0,3,0}}\eta_{{2,1,0}}+{\eta_{{1,0,2}}}^{2}+2\,\eta_{{1,0,2}}\eta_{{1
,2,0}}+2\,\eta_{{1,0,2}}\eta_{{3,0,0}}+{\eta_{{1,2,0}}}^{2}+\\
&+2\,\eta_{{1,2,0}}\eta_
{{3,0,0}}+{\eta_{{2,0,1}}}^{2}+{\eta_{{2,1,0}}}^{2}+{\eta_{{3,0,0}}}^{2},\\
B_3=&{\eta_{{0,0,3}}}^{2}-3\,\eta_{{0,0,3}}\eta_{{0,2,1}}-3\,\eta_{{0,0,3}}\eta_{{2,0,1}}+
6\,{\eta_{{0,1,2}}}^{2}-3\,\eta_{{0,1,2}}\eta_{{0,3,0}}-3\,\eta_{{0,1,2}}\eta_{{2,1,0
}}+6\,{\eta_{{0,2,1}}}^{2}-\\
&{-}3\,\eta_{{0,2,1}}\eta_{{2,0,1}}{+}{\eta_{{0,3,0}}}^{2}{-}3
\,\eta_{{0,3,0}}\eta_{{2,1,0}}{+}6\,{\eta_{{1,0,2}}}^{2}{-}3\,\eta_{{1,0,2}}\eta_{{1,2,0}
}{-}3\,\eta_{{1,0,2}}\eta_{{3,0,0}}+15\,{\eta_{{1,1,1}}}^{2}+\\
&+6\,{\eta_{{1,2,0}}}^{2}
-3\,\eta_{{1,2,0}}\eta_{{3,0,0}}+6\,{\eta_{{2,0,1}}}^{2}+6\,{\eta_{{2,1,0}}}^{2}+{
\eta_{{3,0,0}}}^{2}.
\end{align*}

%The listing here the minimal generating set  doesn't make any sense,  but 
All of  the 196  invariants  can be obtained in  a similar way as above.
 
In the book \cite{FSB}, the   3D  moment invariants  $\Phi_1, \ldots, \Phi_{13}$ were presented, in particular, the first degree invariant  $\Phi_1$ and the invariants  $\Phi_2, \Phi_4, \Phi_5$  of  degree two. These invariants could be expressed in terms of the invariants $B_0, B_1, B_2, B_3$ as follows:
\begin{align*}
&\Psi_1=B_0,
\Phi_2=\frac{B_0^2+2 B_1}{3}, 
\Phi_4=\frac{3B_2+2 B_3}{5},\Phi_5=B_2.
\end{align*}

The Poincar\'e series of the algebra  $\mathbb{C}[T_3]^{\mathfrak{sl}_2}$ calculated by using  Maple package (see \cite{Pu}) has the form: 
\begin{gather*}
\mathcal{P}(\mathbb{C}[T_3]^{sl_2},z)=\frac{p_{0246}(z)}{(1-z)(1-z^6)(1-z^5)^2(1-z^4)^3(1-z^3)^3(1-z^2)^3}=\\=
1+z+4\,{z}^{2}+8\,{z}^{3}+26\,{z}^{4}+53\,{z}^{5}+146\,{z}^{6}+305\,{z}^{7}+704\,{z}^{8}+1417\,{z}^{9}+\cdots
\end{gather*}
where
\begin{gather*}
p_{0246}(z)={z}^{28}+{z}^{25}+9\,{z}^{24}+13\,{z}^{23}+37\,{z}^{22}+51\,{z}^{21}+
91\,{z}^{20}+119\,{z}^{19}+181\,{z}^{18}+208\,{z}^{17}+\\+277\,{z}^{16}+
283\,{z}^{15}+311\,{z}^{14}+283\,{z}^{13}+277\,{z}^{12}+208\,{z}^{11}+
181\,{z}^{10}+119\,{z}^{9}+91\,{z}^{8}+51\,{z}^{7}+\\+37\,{z}^{6}+13\,{z}
^{5}+9\,{z}^{4}+{z}^{3}+1
\end{gather*}
Therefore, the algebra  $\mathbb{C}[\eta]_3^{\mathfrak{sl}_2}$ consists of one invariant of degree 1, namely $B_0$. Also, there exists four  linearly independent invariants of degree  two, namely  $B_0^2,B_1, B_2, B_3,$  eight linearly independent invariants of degree  three etc.

%==========================================================================================

\section{The algebra of rational invariants  $\mathbb{C}(U_d)^{\mathfrak{sl}_2}$. }

%===========================================================================================

%Питання на сьогодні є відкритим і мінімальна  породжуюча система  адгебри $\mathbb{C}[U_d]^{\mathfrak{sl}_2}, d>3$ до цього часу не знайдена. Знайдені вище 196 інваріантів породжують алгебру раціональних інварінтів але вони, очевидно,   не є мінімальною породжуючою системою.
Concidering applications, the rational invariants are more interesting applications
than the polynomial ones. In  the paper  \cite {TST}, a set of  1185 of the 3D rotational moment  invariants up to the sixteenth order was presented. However, these invariants do not form a minimal generating system  and setting a minimal generating system is still remaining an open problem.

In  the following theorem we find the cardinality of a minimal generating  set of the algebra of 3D rational rotation invariants. 
\begin{te}  The number of elements in a minimal generating set of the algebra of the rational invariants  $\mathbb{C}(U_d)^{\mathfrak{sl}_2}, d \geq 2$  is equal to
$$
\binom{d+3}{3}-7.%\frac{\left( d+4 \right)  \left( d-1 \right)}{2}-1.
$$
\end{te}

\begin{proof}

Since the group  $SL(2)$ as an affine variety is three-dimensional one, then,  the transcendence degree of  the field  extension  $\mathbb{C}(U_d)^{SO(3)}/\mathbb{C}$ equals to
$$
{\rm tr\,deg}_{\mathbb{C}} \mathbb{C}(U_d)^{SO(3)}= \dim U_d- \dim SO(3).
$$
Thus, the algebra  $\mathbb{C}(U_d)^{\mathfrak{sl}_2}$ consists of exactly   $\dim U_d- 3$ algebraically independent elements. Taking into account that 
$$
\dim U_d-3 =\sum_{k=2}^d \dim T_d-3=\sum_{k=2}^d \binom{k+2}{2}-3=\binom{d+3}{3}-7,
$$%\frac{\left( d+4 \right)  \left( d-1 \right)}{2}-1=
which is equal to that to be proved.
\end{proof} 
In particular, for $d=2,3$ we have  three and  thirteen invariants, respectively. These results are confirmed by the results in  \cite{FSB}. For $d=2$, it implies  that the algebra $\mathbb{C}(U_2)^{\mathfrak{sl}_2}$ is generated by the invariants $I_1, I_2$ and $ I_3.$ %Ці три інваріванти знайдено вище.

 A system of 13 invariants of   the algebra $\mathbb{C}(U_3)^{\mathfrak{sl}_2}$ was presented in  \cite{FSB}.

The authors claim, without proof, that these invariants are \textit{independent}. Below we present another system of thirteen invariants for 
    $\mathbb{C}(U_3)^{\mathfrak{sl}_2}$  and prove that all these invariants form a minimal generating set of the algebra rational invariants$\mathbb{C}(U_3)^{\mathfrak{sl}_2}.$

%

%Представимо 13  інваріанти  і доведемо, що вони алгебраїчно незалежні.
  In the Sect.~3.3  we found an explicit form for each of the  thirteen polynomial invariants of the algebra $\mathbb{C}[U_3]^{\mathfrak{sl}_2}$.%  Now we prove that the invariants form a minimal generating set of the algebra rational invariants$\mathbb{C}(U_3)^{\mathfrak{sl}_2}.$
	Though, the expressions for the invariants are quite cumbersome, we  will express them    in terms of  the eigenvectors of the  Laplace operator
 $\mathcal{L}$. The operator $\mathcal{L}$ acts on the  basis of  $T^*_3$  as follows:

\begin{align*}
&\mathcal{L}(a_{{0,0,3}})=-12\,a_{{2,0,1}}+12\,a_{{0,0,3}}-12\,a_{{0,2,1}},\mathcal{L}(a_{{0,1,2}})=20\,a_{{0,1,2}}-4\,a_{{2,1,0}}-4\,a_{{0,3,0}},\\
&\mathcal{L}(a_{{0,2,1}})=20\,a_{{0
,2,1}}-4\,a_{{2,0,1}}-4\,a_{{0,0,3}},\mathcal{L}(a_{{0,3,0}})=-12\,a_{{2,1,0}}+12\,
a_{{0,3,0}}-12\,a_{{0,1,2}},\\
&\mathcal{L}(a_{{1,0,2}})=20\,a_{{1,0,2}}-4\,a_{{3,0,0}}
-4\,a_{{1,2,0}},\mathcal{L}(a_{{1,1,1}})=24\,a_{{1,1,1}},\\
&\mathcal{L}(a_{{1,2,0}})=20\,a_{{1,2,0}
}-4\,a_{{3,0,0}}-4\,a_{{1,0,2}},\mathcal{L}(a_{{2,0,1}})=-4\,a_{{0,2,1}}+20\,a_{{2,0
,1}}-4\,a_{{0,0,3}},\\
& \mathcal{L}(a_{{2,1,0}})=-4\,a_{{0,3,0}}+20\,a_{{2,1,0}}-4\,a_{
{0,1,2}},\mathcal{L}(a_{{3,0,0}})=-12\,a_{{1,2,0}}+12\,a_{{3,0,0}}-12\,a_{{1,0,2}}
\end{align*}

Let us recall that $T^*_3 =  \mathcal{V}_2(y_0)  \oplus \mathcal{V}_6(u_0)$.
Let  $c_1,c_2, c_3$ and $b_1,b_2, b_3,b_4,b_4, b_5, b_7$  denote the eigenvectors of $\mathcal{L}$ in the vector spaces $\mathcal{V}_2(y_0)$ and $\mathcal{V}_6(u_0)$, respectively.
We find the eigenvectors by the standard linear algebra algorithm:
\begin{align*}
&c_{{1}}=a_{{0,0,3}}+a_{{0,2,1}}+a_{{2,0,1}},c_{{2}}=a_{{0
,1,2}}+a_{{0,3,0}}+a_{{2,1,0}},c_{{3}}=a_{{1,0,2}}+a_{{1,2,0}}+a_{{3,0
,0}}\\
&b_{{1}}=a_{{0,0,3}}-3\,a_{{0,2,1}},b_{{2}}=-3\,a_{{0,1,2}}+a_{{0,3,0}},b_{{3}}=a_{{1,1,1}},\\ 
&b_{{4}}=a_{{1,2,0}}-a_{{1,0,2}},b_{{5}}=a_{{2,0,
1}}-a_{{0,2,1}},b_{{6}}=a_{{2,1,0}}-a_{{0,1,2}},b_{{7}}=-3\,a_{{1,0,2}
}+a_{{3,0,0}}.
\end{align*}
The eigenvectors for the spaces  $\mathcal{V}_0(u_0)$ and $\mathcal{V}_4(x_0)$ we already found in Subsect.~\ref{ss2}.
Now, let us express the above thirteen invariants in terms of the eigenvectors.  We have
\begin{gather*}
\mathbi{od}=e_0, \\
\mathbi{dv}_1={c_{{1}}}^{2}+{c_{{2}}}^{2}+{c_{{3}}}^{2},\\
\mathbi{dv}_2=3\,{e_{{1}}}^{2}+{e_{{2}}}^{2}-e_{{5}}e_{{2}}+3\,{e_{{3}}}^{2}+3\,{e_{
{4}}}^{2}+{e_{{5}}}^{2},\\
\mathbi{dv}_3={b_{{1}}}^{2}-3\,b_{{5}}b_{{1}}+{b_{{2}}}^{2}-3\,b_{{6}}b_{{2}}+15\,{b
_{{3}}}^{2}+6\,{b_{{4}}}^{2}-3\,b_{{4}}b_{{7}}+6\,{b_{{5}}}^{2}+6\,{b_
{{6}}}^{2}+{b_{{7}}}^{2},\\
\mathbi{tr}_1=9\,{e_{{1}}}^{2}e_{{2}}-18\,{e_{{1}}}^{2}e_{{5}}+54\,e_{{1}}e_{{4}}e_{
{3}}+2\,{e_{{2}}}^{3}-3\,e_{{5}}{e_{{2}}}^{2}-18\,{e_{{3}}}^{2}e_{{2}}
+9\,{e_{{4}}}^{2}e_{{2}}-3\,{e_{{5}}}^{2}e_{{2}}+9\,{e_{{3}}}^{2}e_{{5
}}+\\+9\,{e_{{4}}}^{2}e_{{5}}+2\,{e_{{5}}}^{3},\\
\mathbi{tr}_2={c_{{1}}}^{2}e_{{2}}+{c_{{1}}}^{2}e_{{5}}-6\,e_{{1}}c_{{1}}c_{{2}}-6\,
e_{{3}}c_{{3}}c_{{1}}-2\,{c_{{2}}}^{2}e_{{2}}+{c_{{2}}}^{2}e_{{5}}-6\,
e_{{4}}c_{{3}}c_{{2}}+{c_{{3}}}^{2}e_{{2}}-2\,{c_{{3}}}^{2}e_{{5}},\\
\mathbi{tr}_3=2\,{b_{{1}}}^{2}e_{{2}}+2\,{b_{{1}}}^{2}e_{{5}}+3\,e_{{1}}b_{{2}}b_{{1
}}+60\,e_{{4}}b_{{3}}b_{{1}}+3\,e_{{3}}b_{{1}}b_{{4}}-21\,b_{{1}}b_{{5
}}e_{{2}}+9\,b_{{1}}b_{{5}}e_{{5}}+\\
+3\,e_{{1}}b_{{6}}b_{{1}}+3\,e_{{3}}
b_{{1}}b_{{7}}-4\,{b_{{2}}}^{2}e_{{2}}+2\,{b_{{2}}}^{2}e_{{5}}+60\,b_{
{2}}b_{{3}}e_{{3}}-12\,e_{{4}}b_{{2}}b_{{4}}+3\,e_{{1}}b_{{2}}b_{{5}}+
12\,b_{{2}}b_{{6}}e_{{2}}+\\+9\,b_{{2}}b_{{6}}e_{{5}}+3\,e_{{4}}b_{{2}}b_
{{7}}-90\,b_{{3}}e_{{1}}b_{{4}}-90\,e_{{4}}b_{{3}}b_{{5}}-90\,b_{{3}}b
_{{6}}e_{{3}}+60\,b_{{3}}e_{{1}}b_{{7}}-18\,{b_{{4}}}^{2}e_{{2}}-9\,{b
_{{4}}}^{2}e_{{5}}+\\
+63\,e_{{3}}b_{{5}}b_{{4}}-27\,e_{{4}}b_{{6}}b_{{4}}
+9\,b_{{4}}b_{{7}}e_{{2}}+12\,b_{{4}}b_{{7}}e_{{5}}+27\,{b_{{5}}}^{2}e
_{{2}}-18\,{b_{{5}}}^{2}e_{{5}}-72\,e_{{1}}b_{{6}}b_{{5}}-12\,e_{{3}}b
_{{5}}b_{{7}}-\\-9\,{b_{{6}}}^{2}e_{{2}}-18\,{b_{{6}}}^{2}e_{{5}}-12\,e_{
{4}}b_{{6}}b_{{7}}+2\,{b_{{7}}}^{2}e_{{2}}-4\,{b_{{7}}}^{2}e_{{5}},\\
\mathbi{tr}_4=c_{{1}}b_{{1}}e_{{2}}+c_{{1}}b_{{1}}e_{{5}}+2\,e_{{1}}c_{{2}}b_{{1}}+2
\,e_{{3}}c_{{3}}b_{{1}}+2\,e_{{1}}c_{{1}}b_{{2}}-2\,c_{{2}}b_{{2}}e_{{
2}}+c_{{2}}b_{{2}}e_{{5}}+2\,e_{{4}}c_{{3}}b_{{2}}-\\
-10\,e_{{4}}b_{{3}}c
_{{1}}-10\,b_{{3}}c_{{2}}e_{{3}}-10\,b_{{3}}e_{{1}}c_{{3}}+2\,e_{{3}}c
_{{1}}b_{{4}}-8\,e_{{4}}c_{{2}}b_{{4}}-4\,b_{{4}}c_{{3}}e_{{2}}+3\,b_{
{4}}c_{{3}}e_{{5}}+b_{{5}}c_{{1}}e_{{2}}-\\
-4\,b_{{5}}c_{{1}}e_{{5}}+2\,e
_{{1}}c_{{2}}b_{{5}}-8\,e_{{3}}c_{{3}}b_{{5}}+2\,e_{{1}}c_{{1}}b_{{6}}
+3\,b_{{6}}c_{{2}}e_{{2}}-4\,b_{{6}}c_{{2}}e_{{5}}-\\-8\,e_{{4}}c_{{3}}b_
{{6}}+2\,e_{{3}}c_{{1}}b_{{7}}+2\,e_{{4}}c_{{2}}b_{{7}}+c_{{3}}b_{{7}}
e_{{2}}-2\,c_{{3}}b_{{7}}e_{{5}},\\
\mathbi{ch}_1=2\,b_{{1}}{c_{{1}}}^{3}-3\,{c_{{2}}}^{2}c_{{1}}b_{{1}}-3\,b_{{1}}{c_{{
3}}}^{2}c_{{1}}-3\,c_{{2}}b_{{2}}{c_{{1}}}^{2}+2\,b_{{2}}{c_{{2}}}^{3}
-3\,b_{{2}}{c_{{3}}}^{2}c_{{2}}+30\,b_{{3}}c_{{2}}c_{{3}}c_{{1}}-\\
-3\,{c
_{{1}}}^{2}b_{{4}}c_{{3}}+12\,{c_{{2}}}^{2}b_{{4}}c_{{3}}-3\,b_{{4}}{c
_{{3}}}^{3}-3\,b_{{5}}{c_{{1}}}^{3}-3\,{c_{{2}}}^{2}c_{{1}}b_{{5}}+12
\,b_{{5}}{c_{{3}}}^{2}c_{{1}}-3\,b_{{6}}c_{{2}}{c_{{1}}}^{2}-3\,b_{{6}
}{c_{{2}}}^{3}+\\+12\,b_{{6}}{c_{{3}}}^{2}c_{{2}}-3\,{c_{{1}}}^{2}c_{{3}}
b_{{7}}-3\,{c_{{2}}}^{2}c_{{3}}b_{{7}}+2\,{c_{{3}}}^{3}b_{{7}},
\end{gather*}
\begin{gather*}
\mathbi{ch}_2={b_{{1}}}^{2}{c_{{1}}}^{2}-3\,{c_{{2}}}^{2}{b_{{1}}}^{2}-3\,{b_{{1}}}^
{2}{c_{{3}}}^{2}-2\,c_{{2}}b_{{2}}c_{{1}}b_{{1}}-40\,b_{{3}}c_{{2}}c_{
{3}}b_{{1}}-2\,b_{{1}}b_{{4}}c_{{3}}c_{{1}}-3\,b_{{1}}b_{{5}}{c_{{1}}}
^{2}+\\
+19\,{c_{{2}}}^{2}b_{{1}}b_{{5}}-b_{{1}}{c_{{3}}}^{2}b_{{5}}-2\,b_
{{6}}c_{{2}}c_{{1}}b_{{1}}-2\,b_{{1}}c_{{3}}b_{{7}}c_{{1}}-3\,{b_{{2}}
}^{2}{c_{{1}}}^{2}+{b_{{2}}}^{2}{c_{{2}}}^{2}-3\,{b_{{2}}}^{2}{c_{{3}}
}^{2}-40\,b_{{3}}b_{{2}}c_{{3}}c_{{1}}+\\
+8\,b_{{2}}b_{{4}}c_{{3}}c_{{2}}
-2\,c_{{2}}b_{{2}}c_{{1}}b_{{5}}+19\,b_{{6}}b_{{2}}{c_{{1}}}^{2}-3\,b_
{{2}}b_{{6}}{c_{{2}}}^{2}-b_{{2}}{c_{{3}}}^{2}b_{{6}}-2\,b_{{2}}c_{{3}
}b_{{7}}c_{{2}}-25\,{b_{{3}}}^{2}{c_{{1}}}^{2}-\\
-25\,{b_{{3}}}^{2}{c_{{2
}}}^{2}-25\,{b_{{3}}}^{2}{c_{{3}}}^{2}+60\,b_{{3}}c_{{2}}b_{{4}}c_{{1}
}+60\,b_{{3}}c_{{2}}c_{{3}}b_{{5}}+60\,b_{{3}}b_{{6}}c_{{3}}c_{{1}}-40
\,b_{{3}}c_{{2}}b_{{7}}c_{{1}}-28\,{b_{{4}}}^{2}{c_{{1}}}^{2}+\\
+2\,{c_{{
2}}}^{2}{b_{{4}}}^{2}-4\,{b_{{4}}}^{2}{c_{{3}}}^{2}-42\,b_{{5}}b_{{4}}
c_{{3}}c_{{1}}+18\,b_{{6}}b_{{4}}c_{{3}}c_{{2}}+19\,b_{{4}}b_{{7}}{c_{
{1}}}^{2}-{c_{{2}}}^{2}b_{{4}}b_{{7}}-3\,b_{{4}}{c_{{3}}}^{2}b_{{7}}-4
\,{b_{{5}}}^{2}{c_{{1}}}^{2}-\\
-28\,{c_{{2}}}^{2}{b_{{5}}}^{2}+2\,{b_{{5}
}}^{2}{c_{{3}}}^{2}+48\,b_{{6}}c_{{2}}b_{{5}}c_{{1}}+8\,b_{{5}}c_{{3}}
b_{{7}}c_{{1}}-28\,{b_{{6}}}^{2}{c_{{1}}}^{2}-4\,{b_{{6}}}^{2}{c_{{2}}
}^{2}+2\,{b_{{6}}}^{2}{c_{{3}}}^{2}+\\+8\,b_{{6}}c_{{3}}b_{{7}}c_{{2}}-3
\,{b_{{7}}}^{2}{c_{{1}}}^{2}-3\,{c_{{2}}}^{2}{b_{{7}}}^{2}+{c_{{3}}}^{
2}{b_{{7}}}^{2},
\end{gather*}
\begin{align*}
\mathbi{ch}_3=&b_{{1}}{e_{{1}}}^{2}e_{{4}}-b_{{1}}e_{{1}}e_{{2}}e_{{3}}+b_{{1}}e_{{1}
}e_{{3}}e_{{5}}-b_{{1}}{e_{{3}}}^{2}e_{{4}}-b_{{2}}{e_{{1}}}^{2}e_{{3}
}-b_{{2}}e_{{1}}e_{{4}}e_{{5}}+b_{{2}}e_{{3}}{e_{{4}}}^{2}-\\
&-b_{{3}}{e_{{2}}}^{2}e_{{5}}+b_{{3}}e_{{2}}{e_{{4}}}^{
2}+b_{{3}}e_{{2}}{e_{{5}}}^{2}+b_{{3}}{e_{{3}}}^{2}e_{{5}}-b_{{3}}{e_{
{4}}}^{2}e_{{5}}+b_{{4}}{e_{{1}}}^{3}+b_{{4}}e_{{1}}e_{{2}}e_{{5}}-2\,
b_{{4}}e_{{1}}{e_{{3}}}^{2}+b_{{4}}e_{{1}}{e_{{4}}}^{2}-\\
&-b_{{4}}e_{{1}}
{e_{{5}}}^{2}-2\,b_{{4}}e_{{2}}e_{{3}}e_{{4}}+b_{{4}}e_{{3}}e_{{4}}e_{
{5}}-2\,b_{{5}}{e_{{1}}}^{2}e_{{4}}+b_{{5}}e_{{1}}e_{{2}}e_{{3}}-2\,b_
{{5}}e_{{1}}e_{{3}}e_{{5}}-b_{{5}}e_{{2}}e_{{4}}e_{{5}}+\\
&+b_{{5}}{e_{{3}
}}^{2}e_{{4}}+b_{{5}}{e_{{4}}}^{3}+2\,b_{{6}}{e_{{1}}}^{2}e_{{3}}-b_{{
6}}e_{{1}}e_{{2}}e_{{4}}+2\,b_{{6}}e_{{1}}e_{{4}}e_{{5}}+b_{{6}}{e_{{2
}}}^{2}e_{{3}}-b_{{6}}e_{{2}}e_{{3}}e_{{5}}-b_{{6}}{e_{{3}}}^{3}+\\&+b_{{7}}e_{{1}}{e_{{3}}}^{2}-b_{{7}}e_{{1}}{e_{{
4}}}^{2}+b_{{7}}e_{{2}}e_{{3}}e_{{4}}-b_{{6
}}e_{{3}}{e_{{4}}}^{2},
\end{align*}
\begin{align*}
\mathbi{ch}_4=&2\,{c_{{1}}}^{2}{e_{{2}}}^{2}-5\,{c_{{1}}}^{2}e_{{2}}e_{{5}}+9\,{c_{{1
}}}^{2}{e_{{4}}}^{2}+2\,{c_{{1}}}^{2}{e_{{5}}}^{2}-6\,e_{{1}}c_{{1}}c_
{{2}}e_{{2}}+12\,e_{{1}}c_{{1}}c_{{2}}e_{{5}}-18\,c_{{1}}c_{{2}}e_{{3}
}e_{{4}}-\\&
-18\,c_{{1}}c_{{3}}e_{{1}}e_{{4}}+12\,e_{{3}}c_{{3}}c_{{1}}e_{
{2}}-6\,e_{{3}}c_{{3}}c_{{1}}e_{{5}}-{c_{{2}}}^{2}{e_{{2}}}^{2}+{c_{{2
}}}^{2}e_{{2}}e_{{5}}+9\,{c_{{2}}}^{2}{e_{{3}}}^{2}+2\,{c_{{2}}}^{2}{e
_{{5}}}^{2}-\\&-18\,c_{{2}}c_{{3}}e_{{1}}e_{{3}}-6\,e_{{4}}c_{{3}}c_{{2}}e
_{{2}}-6\,e_{{4}}c_{{3}}c_{{2}}e_{{5}}+9\,{c_{{3}}}^{2}{e_{{1}}}^{2}+2
\,{c_{{3}}}^{2}{e_{{2}}}^{2}+{c_{{3}}}^{2}e_{{2}}e_{{5}}-{c_{{3}}}^{2}
{e_{{5}}}^{2},
\end{align*}
\begin{gather*}
\mathbi{ch}_5={b_{{1}}}^{4}-6\,{b_{{1}}}^{3}b_{{5}}+7\,{b_{{2}}}^{2}{b_{{1}}}^{2}-36
\,b_{{6}}b_{{2}}{b_{{1}}}^{2}-50\,{b_{{3}}}^{2}{b_{{1}}}^{2}+57\,{b_{{
4}}}^{2}{b_{{1}}}^{2}-36\,b_{{4}}b_{{7}}{b_{{1}}}^{2}+{b_{{5}}}^{2}{b_
{{1}}}^{2}+\\+57\,{b_{{6}}}^{2}{b_{{1}}}^{2}+7\,{b_{{7}}}^{2}{b_{{1}}}^{2
}-36\,{b_{{2}}}^{2}b_{{5}}b_{{1}}+60\,b_{{3}}b_{{2}}b_{{4}}b_{{1}}-40
\,b_{{3}}b_{{2}}b_{{7}}b_{{1}}+128\,b_{{6}}b_{{2}}b_{{5}}b_{{1}}+150\,
{b_{{3}}}^{2}b_{{5}}b_{{1}}+\\
+160\,b_{{3}}b_{{6}}b_{{4}}b_{{1}}+60\,b_{{
3}}b_{{6}}b_{{7}}b_{{1}}-206\,{b_{{4}}}^{2}b_{{5}}b_{{1}}+88\,b_{{4}}b
_{{7}}b_{{5}}b_{{1}}+24\,{b_{{5}}}^{3}b_{{1}}-136\,{b_{{6}}}^{2}b_{{5}
}b_{{1}}-6\,{b_{{7}}}^{2}b_{{5}}b_{{1}}+\\
+{b_{{2}}}^{4}-6\,b_{{6}}{b_{{2
}}}^{3}-50\,{b_{{3}}}^{2}{b_{{2}}}^{2}+12\,{b_{{2}}}^{2}{b_{{4}}}^{2}-
6\,{b_{{2}}}^{2}b_{{4}}b_{{7}}+57\,{b_{{2}}}^{2}{b_{{5}}}^{2}+{b_{{6}}
}^{2}{b_{{2}}}^{2}+7\,{b_{{2}}}^{2}{b_{{7}}}^{2}+\\
+150\,{b_{{3}}}^{2}b_{
{2}}b_{{6}}-340\,b_{{3}}b_{{2}}b_{{4}}b_{{5}}+60\,b_{{3}}b_{{2}}b_{{7}
}b_{{5}}+4\,b_{{2}}{b_{{4}}}^{2}b_{{6}}-52\,b_{{2}}b_{{4}}b_{{7}}b_{{6
}}-136\,b_{{6}}b_{{2}}{b_{{5}}}^{2}+\\
+24\,b_{{2}}{b_{{6}}}^{3}-6\,b_{{2}
}{b_{{7}}}^{2}b_{{6}}+625\,{b_{{3}}}^{4}+200\,{b_{{3}}}^{2}{b_{{4}}}^{
2}+150\,{b_{{3}}}^{2}b_{{7}}b_{{4}}+200\,{b_{{3}}}^{2}{b_{{5}}}^{2}+
200\,{b_{{3}}}^{2}{b_{{6}}}^{2}-\\
-50\,{b_{{3}}}^{2}{b_{{7}}}^{2}-240\,b_
{{3}}b_{{6}}b_{{4}}b_{{5}}+160\,b_{{3}}b_{{6}}b_{{7}}b_{{5}}+16\,{b_{{
4}}}^{4}+24\,{b_{{4}}}^{3}b_{{7}}+137\,{b_{{4}}}^{2}{b_{{5}}}^{2}+17\,
{b_{{6}}}^{2}{b_{{4}}}^{2}+\\
+{b_{{4}}}^{2}{b_{{7}}}^{2}-76\,b_{{4}}b_{{7
}}{b_{{5}}}^{2}+4\,{b_{{6}}}^{2}b_{{4}}b_{{7}}-6\,{b_{{7}}}^{3}b_{{4}}
+16\,{b_{{5}}}^{4}+32\,{b_{{6}}}^{2}{b_{{5}}}^{2}+12\,{b_{{7}}}^{2}{b_
{{5}}}^{2}+\\+16\,{b_{{6}}}^{4}+12\,{b_{{6}}}^{2}{b_{{7}}}^{2}+{b_{{7}}}^
{4}
\end{gather*}

\begin{te} The set of the following thirteen invariants 
$$
{od, dv_1,dv_2,dv_3,tr_1,tr_2,tr_3,tr_4, ch_1,ch_2,ch_3,ch_4,ch_5}
$$
is a minimal generating set of the algebra $\mathbb{C}(U_3)^{\mathfrak{sl}_2}.$
\end{te}
\begin{proof} It is enough to prove that the elements  are algebraically independed.  Let us consider the Jacobian $13 \times 16$-matrix of the  polynomial set: 
	
	$$ \begin{pmatrix} \displaystyle  \frac{\partial \,{ od} }{\partial e_0}  & \displaystyle \frac{\partial \,{ od}}{\partial e_1}  & \ldots & \displaystyle\frac{\partial \,{ od}}{\partial b_6} & \displaystyle\frac{\partial \,{ od}}{\partial b_7}\\
	\displaystyle \frac{\partial \, dv_1 }{\partial e_0}  & \displaystyle \frac{\partial \,dv_1 }{\partial e_1}  & \ldots & \displaystyle \displaystyle\frac{\partial \,{dv}_1}{\partial b_6} & \displaystyle\frac{\partial \,{ dv}_1}{\partial b_7}\\
	\ldots & \ldots & \ldots & \ldots \\
	\displaystyle  \frac{\partial \,{ ch_4} }{\partial e_0}  & \displaystyle \frac{\partial \,{ ch_4}}{\partial e_1}  & \ldots & \displaystyle\frac{\partial \,{ ch_4}}{\partial b_6} & \displaystyle\frac{\partial \,{ ch_4}}{\partial b_7}\\
	\displaystyle \frac{\partial \, ch_5 }{\partial e_0}  & \displaystyle \frac{\partial \,ch_5 }{\partial e_1}  & \ldots & \displaystyle \displaystyle\frac{\partial \,ch_5}{\partial b_6} & \displaystyle\frac{\partial \, ch_5}{\partial b_7}
		\end{pmatrix}$$
		It  is sufficient to show that 	the rank of the matrix is equal to 13. 
	 	After substituting the following expressions 
	\begin{align*}
	&e_{{0}}=1
,e_{{1}}=1, e_{{2}}=23, e_{{3}}=53, e_{{4}}=97, e_{{5}}=151, b_{{1}}=541, b_{{2}}=661, b_{{3}}=827,\\
&b_{{4}}=1009, b_{{5}}=1193,b_{{6}}=1427, b_{{7}}=1619, c_{{1}}=227, c_{{2}}=311,	c_{{3}}=419,
	\end{align*}
	into the Jacobian matrix, we get a matriх    whose entries are all numbers.	Then, by direct calculation,  we obtain that its rankis equal to thirteen. It implies that the Jacobian matrix has the maximal rank equal to thirteen which proves the theorem.		 
\end{proof}

%Рухаючись цим шляхом можна виділити мінімальні породжуючі системи алгебр інварінтів вищих порядків,  наприклад  мінімальна породжуюча система інваріантів четвертого порядку сткладається із 28 інваріантів.
Applying the same scheme, we can find the minimal generating sets of higher orders, for instance, the minimal generating sets of order four consists of 28 algebraically independent invariants.

\section{Conclusion}
In this article, we reviewed the 3D geometric moment invariants in the terms of the classical invariant theory. 
We divided all invariants into two types by introducing the  notions of the algebras of simultaneous rational and polynomial  rotation invariants $\mathbb{C}[\eta]_d^{SO(3)}$ and  $\mathbb{C}(\eta)_d^{SO(3)}$ up to order $d$  where $\eta$ is a set of normalized moments which are already invariants under the scaling and translations.  
In addition, we proved  that  these algebras  are isomorphic to some classical object of the invariant theory, that is, to the algebras of join  invariants of  binary forms $\mathbb{C}[U_d]^{SL(2)}$ and  $\mathbb{C}(U_d)^{SL(2)}$. 
Further on, we used  Lie infinitesimal method  and reduced the problem of calculating the invariants of the group $SO(3)$ to the equivalent one of calculating the invariants of the  Lie  algebra $\mathfrak{sl}_2.$ From the computational point of view,  it is much more simpler problem dealing with   polynomial derivations. 

In the rational case we count out the cardinality of  the minimal generating set of the algebra  $\mathbb{C}(U_d)^{SL(2)}$  and present such  minimal generating set for
 invariants of the degrees two and three. 
 Also we found  the explicit form of the series of the invariants of the degree one  of an arbitrary order,  which plays an important role in different applications as a low-order moments which are less sensitive to noise than the higher-order ones.

The author  hopes that the results will be useful to the researchers in the fields of image analysis and pattern recognition. 
Though, the geometric moments are not as effective as the orthogonal ones are, the obtained results are of independent theoretical interest.   

As we have seen, in contrast to the 2D  case, there is no satisfactory description of 3D rotational invariants of arbitrary order, and the problem of finding the basis of such  invariants    is hopeless.
In our forthcoming researches, we are going to present another invariant constructions,  which seems to be an effective way of describing of 3D image moments.

\end{document}